\documentclass[letterpaper, 10 pt, conference]{ieeeconf}  

\usepackage{amsmath,amsfonts}
\usepackage{algorithm}
\usepackage{array}
\usepackage{color}
\usepackage{xcolor}
\usepackage[caption=false,font=normalsize,labelfont=sf,textfont=sf]{subfig}
\usepackage{textcomp}
\usepackage{stfloats}
\usepackage{url}
\usepackage{verbatim}
\usepackage{graphicx}
\usepackage{cite}

\newtheorem{lemma}{Lemma}
\newtheorem{remark}{Remark}
\newtheorem{corollary}{Corollary}

\usepackage{algpseudocode}

\usepackage{amssymb}

\DeclareMathOperator*{\argmin}{arg\,min}

\IEEEoverridecommandlockouts                              

\title{\LARGE \bf
Semi-autonomous Teleoperation using Differential Flatness of a Crane Robot for  Aircraft In-Wing Inspection
}

\author{ Wade Marquette, Kyle Schultz, Vamsi Jonnalagadda, Benjamin Wong, Joseph Garbini, and Santosh Devasia
\thanks{
Mechanical Eng. Dept., U. of Washington, Seattle WA 98195, USA
        {\tt\small wm25@uw.edu, kylews@uw.edu, jnvk@uw.edu, bycw@uw.edu , garbini@uw.edu, devasia@uw.edu}}%
\thanks{
Exemption of research was granted by the University of Washington Institutional Review Board (\#STUDY00021303).}%
\thanks{
This is an extended version of an article submitted to IEEE for possible publication. Copyright may be transferred without notice, after which this version may no longer be accessible.}
}

\begin{document}

\maketitle
\thispagestyle{empty}
\pagestyle{empty}

\begin{abstract}\label{sec:abstract}
Visual inspection of confined spaces such as aircraft wings is ergonomically challenging for human mechanics. 
This work presents a novel crane robot that can travel the entire span of the aircraft wing, enabling mechanics to perform inspection from outside of the confined space. 
However, teleoperation of the crane robot can still be a challenge due to the need to avoid obstacles in the workspace and potential oscillations of the camera payload. 
The main contribution of this work is to exploit the differential flatness of the crane-robot dynamics for designing reduced-oscillation, collision-free time trajectories of the camera payload for use in teleoperation.
Autonomous experiments verify the efficacy of removing undesired oscillations by 89\%.
Furthermore, teleoperation experiments demonstrate that the controller eliminated collisions (from 33\% to  0\%)  when
12 participants performed an inspection task   with the use of proposed trajectory selection  when compared to the case without it. Moreover, even discounting the failures due to collisions, the proposed approach  improved task efficiency by 18.7\% when compared to the case without it.
\end{abstract}

\begin{keywords}
Inspection robots, Manufacturing, Robot motion, 
Collision avoidance.
\end{keywords}

\section{Introduction}

Confined-space inspection  is a major aspect of aerospace manufacturing and maintenance, especially inside aircraft wings (where fuel is stored), which are ergonomically challenging,  hazardous environments to work in.
For example, mechanics need to don protective suits and respirators for safety in such spaces, which makes the work cumbersome~\cite{eeloscope}. Moreover, ensuring safety requires regular check-ins from an outside partner. 
These difficulties of operating in confined spaces motivate the development of robotic solutions that allow mechanics to perform their work from outside the confined space.
A challenge with typical robotic inspection solutions is that they require repeated time-consuming installation and removal for each of the many separated internal structure segments (bays, see~Fig. \ref{fig:wing}) of the wing in commercial aircraft architectures.
An additional challenge is that  teleoperation (which takes advantage of human expertise to perform complex tasks) can be slow since it is difficult for humans to manage multiple tasks such as (i)~inspection of the space and (ii)~avoiding obstacles -- especially, in complex confined spaces.
This work presents a robotic inspection system that can move through the entire wing without re-installation in each bay and  develops a control system to aid collision avoidance and thereby, improve teleoperation performance.

\begin{figure}[!t]
\centering
\includegraphics[width=0.475\textwidth]{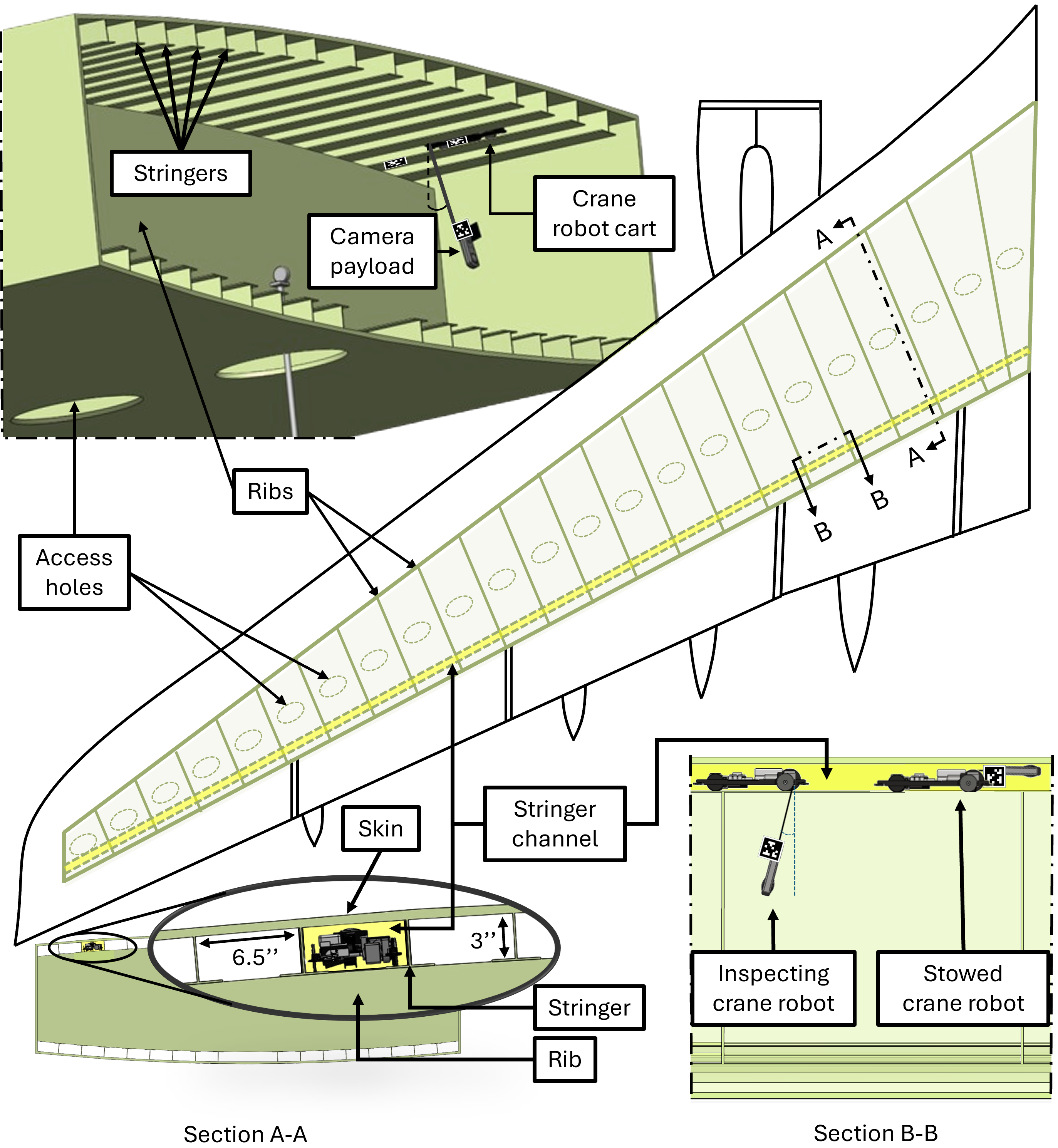}
\caption{
The crane robot performs inspection along the entire wing by traversing in a single stringer channel.
(Top-left) Partial view of an aircraft wing bay  with crane robot moving inside the channel between two stringers, suspending a  camera payload for inspection. The wing bays are partitioned into separate spaces by ribs and 
each space is typically accessed through narrow access holes.
(Center) A two-dimensional schematic of the wing with a sample stringer channel  highlighted in yellow, spanning the length of the wing, which provides access (over the ribs, with the camera stowed) between adjacent  bays.
(Bottom left, Section A-A) Cross section with the stowed crane robot in the stringer channel (highlighted in yellow) formed below the upper skin and between adjacent stringers enabling passage over the ribs.
(Bottom right, Section B-B) Crane robot in the  stringer channel (yellow highlight). 
}
 \vspace{-0.2in}

\label{fig:wing}
\end{figure}

The current work avoids the problem of installation/removal in each of the separate bays by designing a compact robot that can fit and move inside the channel between stringers, and thereby, access the entirety of the wing  operating akin to a gantry crane with the suspended camera as the payload. The camera can be retracted to cross the ribs that separate adjacent bays. Additionally, to aid teleoperation, this work reduces  motion-induced oscillations and automates collision avoidance during teleoperation of the crane robot. 
In particular, the differential flatness of the crane-robot dynamics is used to design
reduced-oscillation, collision-free time trajectories of the camera payload.
The resulting teleoperation controller allows the operator to directly specify camera payload trajectories while autonomously avoiding large oscillations and potential collisions. Experimental evaluations show that the teleoperation assistance reduces undesired oscillation by 89\% and user trials of 12 participants demonstrate the mitigation of collision and an 18.7\% improvement in task completion time when neglecting collision compared to the case without the teleoperation assistance.

\section{Related works}\label{sec:related_works}

\subsection{Design of robots for confined spaces}
Currently available confined space robotic solutions cannot access the entire wing during manufacturing of the commercial aircraft architecture. For example, 
conventional manipulator-based, continuum-type, and snake robot architectures for confined space inspection~\cite{path_inspection},
hole cleaning~\cite{holecleaning}, and multi-tasks~\cite{snake} 
within aircraft wings cannot move between bays through the ribs over the span of the wing and require manual installation and removal for each bay. Movement between bays is achieved with the Eeloscope~\cite{eeloscope} by swimming through holes in the ribs when fuel is present inside aircraft wing tanks. However, this solution is not applicable in the absence of fuel, e.g., during initial aircraft manufacturing. Another approach is using 
mobile robots that drive through rib cutouts in aircraft wings if the driving surfaces are smooth~\cite{driving_hole}, but this solution does not apply to larger aircraft with uneven surfaces due to stiffening stringer structures on the inner-skin.
The crane robot overcomes the challenge of traversing separated bays by using a small cross-sectional area  to traverse the bays through narrow ($6.5$ in $\times$ $3$ in) stringer channels that span the entire wing. Previous works have introduced robots to move in narrow spaces such as  pipes and channels, e.g., snake~\cite{narrow_snake} and inchworm~\cite{pneumatic_worm_pipe} robots.
However, these designs rely on the entire cross section to be continuous, but the stringer channels tend to be open away from the skin. 
Therefore, this work presents a novel crane robot that uses wheels to stay inside the stringer channel, similar to wheeled pipe crawling robots~\cite{wheeled_pipe}. Moreover, the crane robot exploits the opening in the stringer channel to suspend a camera via a pulley mechanism to perform inspection tasks, as shown in Fig.~\ref{fig:wing}. 
In inspection tasks, trajectory tracking is critical to move the payload camera along a desired path, which is different from the goal in traditional gantry cranes used in the aerospace industry that seek to move objects from one point to another.

\subsection{Assisting teleoperation in complex environments}
The crane robot (with controls to position on the stringer channel and the camera position) is analogous to a variable-length gantry crane~\cite{crane_control}. Therefore, it shares similar challenges in teleoperation as industrial crane systems, where the operator controls both the cart position and the cable length independently~\cite{crane_control}.
This conventional control approach, where the operator specifies  trajectories that do not cause undesired oscillations of the payload, can be challenging and require extensive training to learn how to manage the gantry-crane dynamics. To avoid such challenges in teleoperation, this work develops
a semi-autonomous teleoperation control for the crane robot, which can make teleoperation easier as shown in~\cite{semi_auto_obst}.
The human operator only specifies a reference point to generate trajectories for the camera payload, and handling of control complexities such as obstacle avoidance and unwanted oscillations (due to payload dynamics) is managed autonomously to assist teleoperation.
Input shaping is a widely used method for reducing residual oscillations in crane positioning~\cite{input_shaping}.
However, input shaping does not ensure payload trajectory tracking, which is important near obstacles in confined spaces.
Alternatively, gantry cranes have been shown to be differentially flat, allowing all states and inputs to be represented as functions of the output and its time derivatives~\cite{fliess}.
This work leverages the differential flatness property of the crane-robot dynamics as the basis of the teleoperation assistance where camera payload coordinates are considered as the output~\cite{kolar2013,kolar2017}.
Previous work has shown that flatness-based control can reject undesired payload oscillations caused by disturbances during autonomous trajectory tracking and positioning~\cite{yu}.
In addition to crane applications, differential flatness has been widely applied to a variety of systems, including cable-suspended UAV path planning~\cite{cable_UAV}, UAV-UGV cooperative landing~\cite{UAV_UGV}, and aerobatic trajectories of VTOL aircraft~\cite{VTOL}. 
These autonomous applications are based on pre-defined trajectories or objectives but do not
address real-time trajectory generation to assist avoiding obstacles during human teleoperation. Specifically, the proposed crane-robot assistance modulates the operator's reference input to avoid collisions, and additionally, since the trajectory is tracked accurately, the approach also reduces uncontrolled oscillations.
Previous works have used the  flatness-based approach for gantry crane teleoperation using sufficiently-smooth  online S-curve velocity trajectory generation. However, this approach is used in  an environment without obstacles~\cite{craneOnline}. 
Moreover, this approach relies on a fixed-length payload and a linearized model, which may not be sufficiently accurate for crane robot operations, where varying payload lengths and rapid movements can cause larger swing angles.
Therefore, this proposed work develops a trajectory generation approach for the gantry-crane model using the differential flatness of the dynamics, which accounts for both (i) the nonlinearity due to the swing dynamics and (ii) the variable length. 
Thus, the proposed approach enables easier (collision-free and reduced-oscillation), semi-autonomous teleoperation of the crane robot, which in turn, improves operator performance during confined-space inspection.

\section{Crane robot description}\label{sec_system_Des}

To deploy inside the narrow stringer channel formed by the flanges of the stringer and the wing skin, in Fig.~\ref{fig:wing}, the crane robot frame is segmented into two smaller pieces which are subsequently attached (once inside the stringer channel) with a threaded rod. This allows installation through the narrow gap of the stringer channel while being wide enough to drive on the flanges as  as depicted in Fig. \ref{fig:crane_design}(a).
Once installed, the crane robot drives through the channel on four vertical wheels with one wheel driven by a motor.
The crane robot has four side wheels, shown in Fig. \ref{fig:crane_design}(b), which can contact  the stringer webs and help realign the crane robot with the stringer channel.

\begin{figure}[!t]
\centering
\includegraphics[width=0.475\textwidth]{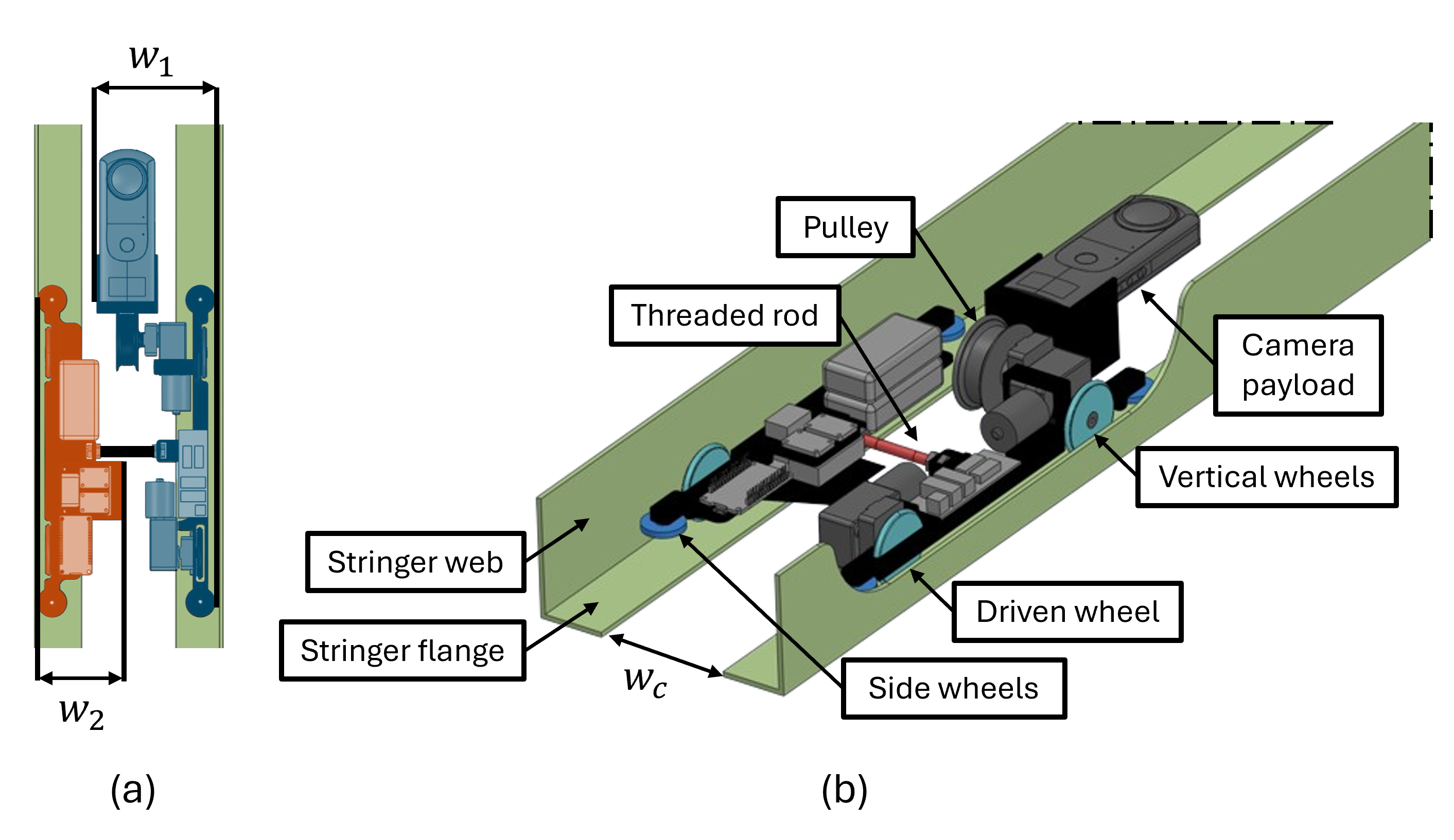}
\caption{
Installation of the crane robot in the stinger channel.
(a) The crane robot frame segmented in two pieces with widths $w_1=3.75$ in and $w_2=2.75$ in and thickness less than $2.5$ in so each segment can be installed diagonally through the channel opening of width $w_c=3.5$ in. (b) The installation is completed by connecting the segmented frame with a threaded rod (red). The width of the connected segments is larger than the channel opening width, $w_c$, enabling the  crane robot to drive on the stringer flanges with its vertical wheels (cyan) while its side wheels (blue) contact the stringer web to correct channel misalignment.
}
\label{fig:crane_design}
\end{figure}

The crane robot performs inspection tasks by suspending a camera into the confined space using a pulley mechanism.
To traverse adjacent bays, the camera payload is stowed and released from the channel by wrapping and unwrapping the camera around the pulley.
The crane robot deploys a wireless 360 camera payload, enabling the operator to perform remote inspection in any direction by panning a tablet application, similar to such use in other applications such as  power-line inspection robots~\cite{linespyx}.

Operators receive overall perspective of the environment using 
an external camera to reduce disorienting effects associated with system movement of the robot's onboard camera~\cite{camera_following}.
Specifically, an external camera, located at the access hole, provides an external view of both the crane robot and obstacles, delivering  situational awareness to the operator shown in Fig. \ref{fig:block_diagram}.
The external camera also provides vision-based state and output  feedback using fiducial markers (Apriltags~\cite{apriltag})
located on the crane-robot's cart and camera payload along with a reference tag located in the confined space
for tracking control, e.g., as in~\cite{concrete_hammer}.
Measurements from the fiducial markers provide swing angle information and global-correction updates for crane-position and payload-length  estimates from motor encoders. 
Operators use  camera feedback to visualize the pose of the crane robot and use a joystick, as in  industrial crane control~\cite{interfaces_study}, to send wireless, horizontal-and-vertical, camera-payload velocity commands.

\begin{figure}[!t]
\centering
\includegraphics[width=0.45\textwidth]{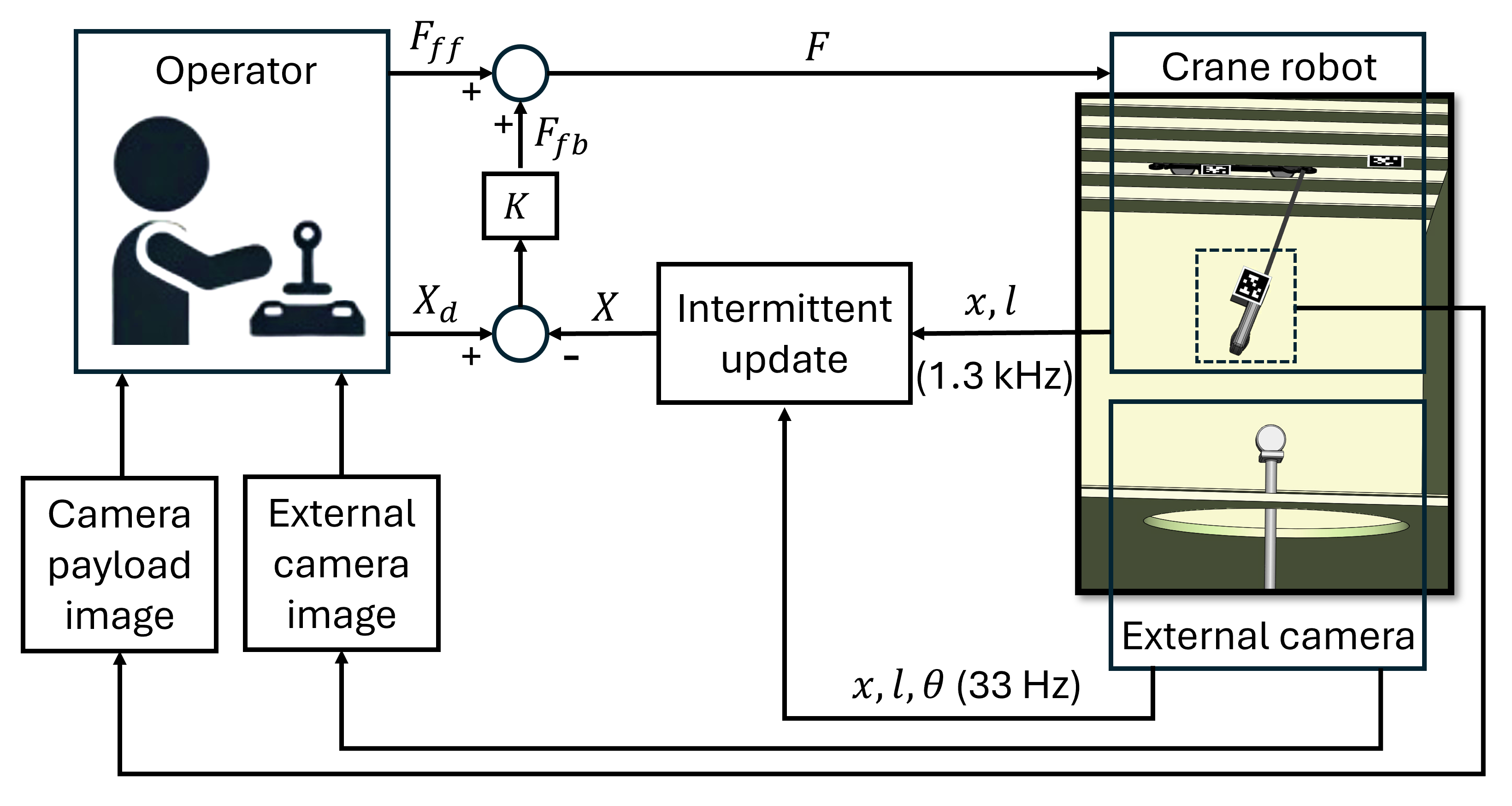}
\caption{
Crane-robot's control scheme.
The external camera at the access hole measures the cart position, $x$, the payload length, $l$, and the swing angle, $\theta$, from the fiducial markers at a sampling rate of 33 Hz, which intermittently updates feedback collected from the motor encoders at a sampling rate of 1.3 kHz to construct the crane robot's states, $X$.
From the operator workstation, the feedforward input, $F_{ff}$, and desired states, $X_d$, are specified through a joystick interface by the operator observing the external camera and camera payload image.
The combined feedforward force, $F_{ff}$, and feedback force, $F_{fb}$, is the applied crane-robot input, $F$.
}
\label{fig:block_diagram}
\end{figure}

Global correction updates use three fiducial markers illustrated in Fig.~\ref{fig:apriltag_diagram} to measure cart position, $x$, camera payload length, $l$, swing angle, $\theta$, and horizontal and vertical camera payload positions, $y_1$ and $y_2$, respectively.
Each fiducial marker returns six degree-of-freedom measurements, with three translations and three rotations from the camera. However, marker translation measurements are more stable than rotation measurements from the access hole camera distance, so states and outputs are computed using marker center coordinates and crane-robot kinematics.
The reference marker defines the origin at coordinates $(0,0)$ with a fixed homogeneous transformation (to remove rotation axis jitter) such that horizontal and vertical distances to markers on the crane robot can be computed through homogeneous transformations as in~\cite{atag_relative}.
Therefore, a marker placed on the camera payload's center of mass directly measures its horizontal and vertical positions of $y_1$ and $y_2$, respectively.
From the measured horizontal distance to the cart marker, $d_1$, and a fixed measurement from the cart marker to the pulley, $d_2$, the cart position, $x$, is measured as $x = d_1 + d_2$.
With the cart position, $x$, the swing angle, $\theta$, can then be measured as $\theta = \arctan\left(\frac{y_1 - x}{y_2}\right)$.
Given the swing angle, $\theta$, the length of the camera payload is $l = -\frac{y_2}{\cos\left( \theta \right)}$.
Measurement noise was quantified as standard deviations when the crane robot was at rest, with the camera payload suspended at the center of the confined space as $\sigma_x = \sigma_{y_1} = 0.0007$ m, $\sigma_l = \sigma_{y_2} =  0.0059$ m, and $\sigma_\theta = 0.0026$ rad.

\begin{figure}[t]
\centering
\includegraphics[width=0.35\textwidth]{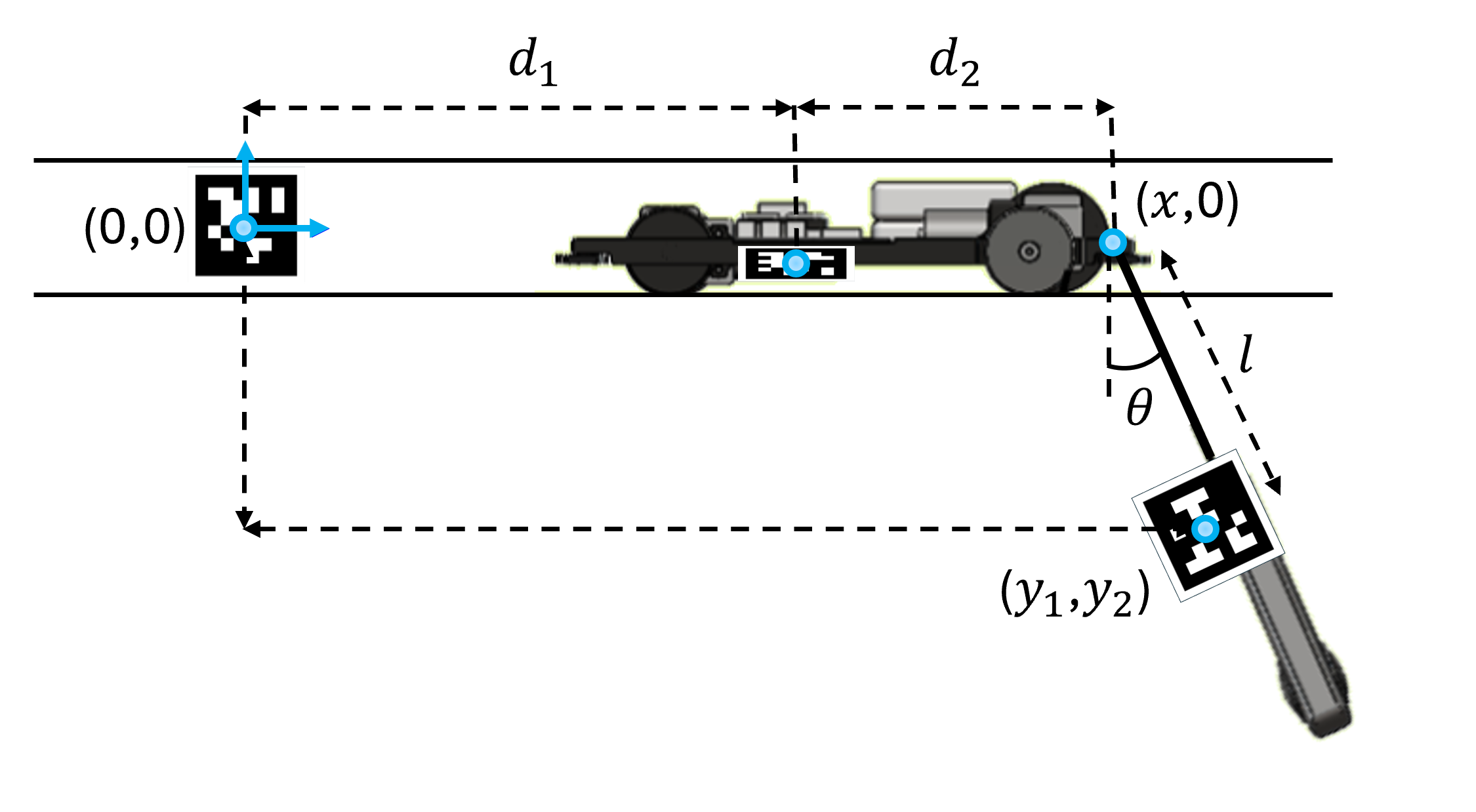}
\caption{
Fiducial marker configuration for global feedback. 
Crane robot states and outputs are computed using a combination of marker coordinate measurements and crane-robot kinematics.
}
\label{fig:apriltag_diagram}
\end{figure}

\section{Control problem formulation}\label{sec:probform}
The dynamics of the crane robot resembles those of a variable-length gantry crane, as depicted in Fig.~\ref{fig:prob_form}(a).
Under low-speed operations, it is feasible to neglect the payload swing, $\theta$, dynamics.
Without the swing dynamics, the horizontal position, $y_{1}$, of the payload (inspection camera)  
corresponds to the cart position, $x$, and the vertical camera position, $y_{2}$,  corresponds to the payload length, $l$. 
As a result, the outputs 
(the horizontal $y_{1}$ and vertical $y_{2}$ positioning of the camera) are  decoupled from each other and can be controlled independent of each other.

The decoupled approach, without compensating for the swing dynamics (i.e. assuming $x = y_{1}$ and $l = -y_{2}$), can lead to acceptable tracking during low-speed operation. To illustrate, tracking is studied for a ramp-like trajectory to move across the confined space over a pipeline obstacle, as shown in Fig. \ref{fig:prob_form}(b).
The nominal desired trajectory, 
$\tilde{Y} = \begin{bmatrix} \tilde{y}_1(\tau) & \tilde{y}_2(\tau) \end{bmatrix}^T$, is defined by 

\vspace{-0.1in} 
{\small{
\begin{equation}
    \tilde{Y} = 
    \begin{cases}
        \begin{bmatrix}
            0 & 0
        \end{bmatrix}^T & 0 \leq \tau < T_t \\
        \begin{bmatrix}
        \frac{\Delta y_1}{T_t}(\tau - T_t) & \frac{2\Delta y_2}{T_t}(\tau - T_t)
        \end{bmatrix}^T & T_t \leq \tau < \frac{3}{2}T_t \\
        \begin{bmatrix}
        \frac{\Delta y_1}{T_t}(\tau - T_t) & -\frac{2\Delta y_2}{T_t}(\tau - 2T_t) 
        \end{bmatrix}^T & \frac{3}{2}T_t \leq \tau < 2T_t \\
        \begin{bmatrix}
        \Delta y_1 & 0 
        \end{bmatrix}^T & 2T_t \leq \tau < 4T_t \\

    \end{cases} 
    ,
    \notag
\end{equation}
}}

%
%

\vspace{-0.1in} 
\noindent 
where the transition time, $T_t$, defines the duration of movement, $\Delta y_1$ defines the change in the camera's horizontal position over $T_t$, and $\Delta y_2$ defines the change in the camera's vertical position over $\frac{T_t}{2}$ before returning to the initial position within $\frac{T_t}{2}$.
The nominal time trajectories $\tilde{y}_1$ and $\tilde{y}_2$ are smoothed by four cascaded first-order low-pass filters with a cutoff frequency, $\alpha$, as $y_1(\tau)=F_f\tilde{y_1}$ and $y_2(\tau)=F_f\tilde{y_2}$, with
$   F_f = \left( \frac{\alpha}{\lambda+\alpha} \right)^4$
where $\lambda$ represents the Laplace variable.

The response of the 
crane robot with cart mass, $M=0.815$ kg, and payload mass, $m=0.225$ kg, when following the time 
trajectories $y_1(\tau)$ and $y_2(\tau)$, parameterized by $\Delta y_1$ = 0.75 m, $\Delta y_2$ = 0.15 m, and $\alpha$ = 10 $\frac{\text{rad}}{\text{s}}$ (1.59 Hz) at a slow transition time of $T_t$ of 40 seconds, by tracking the decoupled commands of the cart and payload length are shown in Fig. \ref{fig:prob_form}(c).
Residual oscillations, present after reaching the inspection point, are small and remain below a magnitude of 1.4 degrees, as shown in Fig. \ref{fig:prob_form}(d).

\begin{figure}[!ht]
\centering
\includegraphics[width=0.5\textwidth]{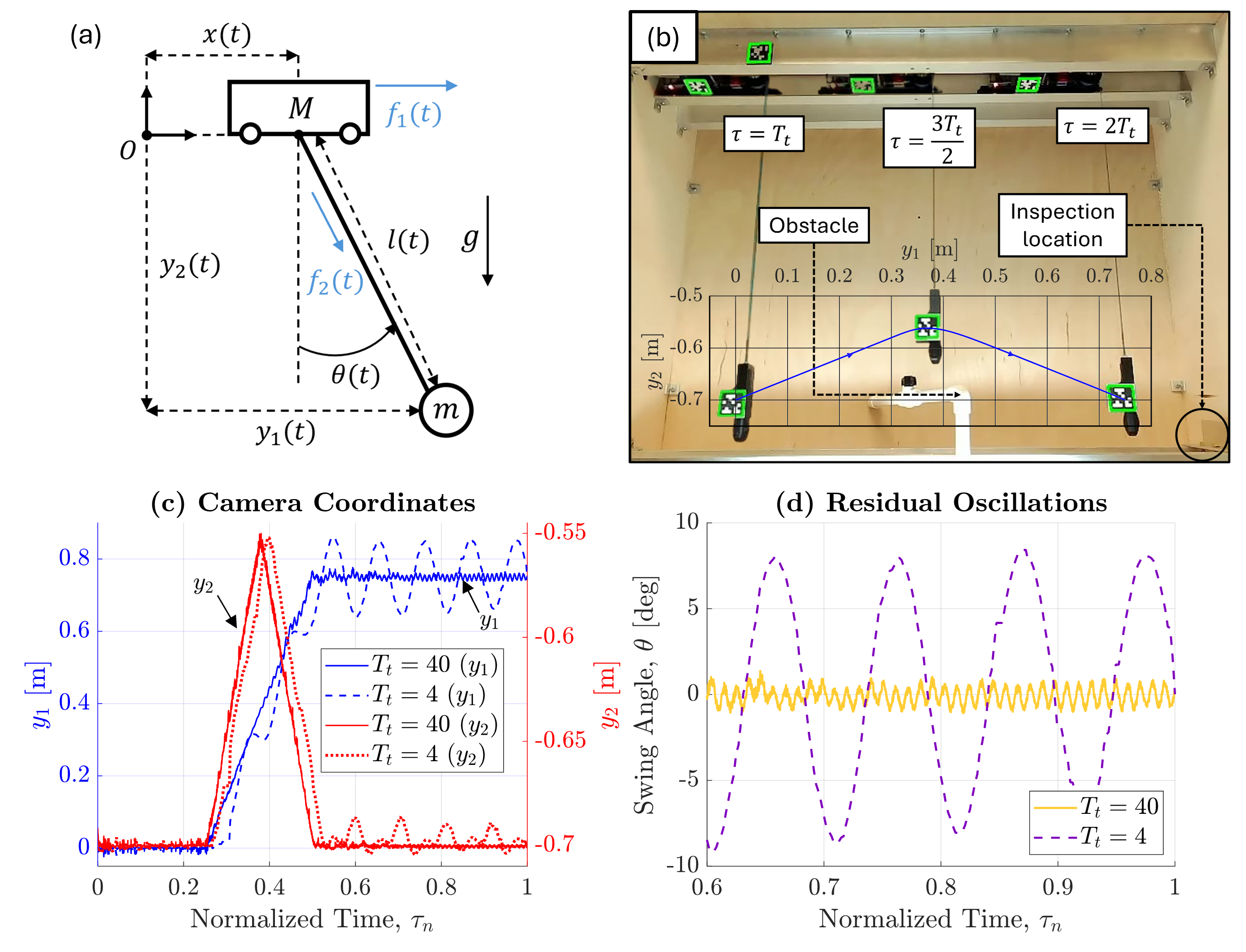}
\caption{
Comparison of residual oscillations  for slow and fast 
trajectories,  without  swing-dynamics compensation.
(a)~The crane-robot schematic. The system inputs are the force on the cart, $f_1$, and the force on the payload, $f_2$, and the system outputs are the horizontal and vertical positions of the payload, $y_1$ and $y_2$, respectively. 
The states of the system are the cart's position, $x$, the payload length, $l$, and the payload swing angle, $\theta$, along with their time derivatives.
(b)~The experimental confined space with snapshots of the crane robot prototype traversing the ramp-like trajectory which transitions the camera payload over a pipeline obstacle near an inspection location. 
(c)~Comparative time responses of tracking the decoupled commands for the horizontal camera coordinate, $y_1$, and the vertical camera coordinate, $y_2$, at two different transition times, $T_t=4$ and $T_t=40$ seconds, plotted against normalized time, $\tau_n = \frac{\tau}{4T_t}$. 
(d)~The residual oscillations upon arriving at the inspection point for the two transition times.
}
\vspace{-0.1in} 
\label{fig:prob_form}
\end{figure}

\noindent 
However, this decoupling of the payload positioning is only valid at low speeds and accelerations, and rapid changes in the positions (e.g., needed for faster movements to speed up inspection) can excite the swing dynamics, inducing significant oscillations, which in turn, can make  teleoperation challenging.
The swing dynamics are excited as the same ramp-like trajectory is tracked with a smaller transition time, $T_t=4$ seconds, with a 6.5 times increase in residual oscillation magnitude to 9.1 degrees, as shown in Fig. \ref{fig:prob_form}(c)(d).
Such large residual oscillations need to be avoided to enable fast teleoperation. 
Therefore, the research problem is to compensate for the swing dynamics to reduce residual oscillations.

\section{Flatness-based semi-autonomous control}\label{sec:methods}

\subsection{Flatness-based feedforward inputs}

The dynamics of the crane robot can be modeled as~\cite{yu}
{\small{
\begin{equation}\label{eq:sys_dyn} 
\begin{bmatrix} 
(M+m) & mlc &  ms \\
mlc & ml^2 & 0 \\ 
ms & 0 & m
\end{bmatrix} 
\begin{bmatrix} 
\ddot{x} \\ \ddot{\theta} \\ \ddot{l}
\end{bmatrix} 
 = 
\begin{bmatrix} 
 m l \dot{\theta}^2 s - 2 m \dot{l} \dot{\theta} c   +f_1 \\
  - 2 m l \dot{l} \dot{\theta} - m g l s \\
 m l \dot{\theta}^2 + m g c +f_2
\end{bmatrix}  ,
\end{equation}
}}

\noindent where $M$ is the mass of the cart, $m$ is the mass of the payload, $f_1$ is the cart force, $f_2$ is the payload force, $g$ is the gravitational acceleration, $s=\sin(\theta(t))$,  $c=\cos(\theta(t))$, and it is assumed that the rolling  friction on the cart is negligible, or has been compensated. 
Inverting the square matrix on the left hand side  of Eq.~\eqref{eq:sys_dyn}, and rewriting in the state-space form, results in 

{\small{
\begin{align}
\label{eq_sys_dyn_isolated} 
\frac{d}{dt} X
& = 
\begin{bmatrix} 
\dot{x} \\
 0 \\
 \dot{\theta} \\
 - \frac{2}{l}\dot{l}\dot{\theta} - \frac{g}{l}s \\
  \dot{l} \\
l \dot{\theta}^2 + g c \\
\end{bmatrix}  + 
\begin{bmatrix} 
0 & 0 \\
 \frac{1}{M} &   \frac{-s}{M} \\
 0 & 0 \\
  \frac{-c}{M l} &  \frac{s c}{M l}
  \\
  0 & 0 \\
 \frac{-s}{M} & \left( \frac{s^2}{M} + \frac{1}{m}\right)
\end{bmatrix} 
\begin{bmatrix} 
 f_1 \\
 f_2
\end{bmatrix},
\\
& = 
f(X) + g(X)F
\end{align} 
}}

\vspace{-0.1in}
\noindent with state vector $X = \begin{bmatrix} x & \dot{x} & \theta & \dot{\theta} & l & \dot{l} \text{\,} \end{bmatrix}^T$.
The outputs of the system are the camera's horizontal position, $y_1$, and vertical position, $y_2$, which can be 
expressed in terms of the crane-robot states (${x, l, \theta}$) as
\begin{equation}\label{eq:y}
Y~=~
\begin{bmatrix} y_1 \\ y_2 \end{bmatrix} = 
\begin{bmatrix} x+ls \\ -lc \end{bmatrix}
.
\end{equation}

\vspace{-0.05in}
\noindent 
To enable tracking of the outputs ($y_1, y_2$), an expression relating the input forces 
($f_1$ and $f_2$) to the outputs is found by  differentiating the outputs until the inputs appears~\cite{kolar2017}. Specifically, differentiating Eq.~\eqref{eq:y} twice results in
\begin{equation}\label{eq_ydot_X}
    \dot{Y}~=~
    \begin{bmatrix} \dot{y}_1 \\ \dot{y}_2 \end{bmatrix} = 
    \begin{bmatrix} \dot{x} + \dot{l}s + l\dot{\theta}c \\
    -\dot{l}c+l\dot{\theta}s    
    \end{bmatrix}
    ,
\end{equation}
\begin{equation}\label{eq_yddot_X}
    \ddot{Y}~=~
    \begin{bmatrix} \ddot{y}_1 \\ \ddot{y}_2 \end{bmatrix} = 
    \begin{bmatrix} \ddot{x} + \ddot{l}s + 2\dot{l}\dot{\theta}c + l\ddot{\theta}c - l\dot{\theta}^2s \\ -\ddot{l}c + 2\dot{l}\dot{\theta}s + l\ddot{\theta}s + l\dot{\theta}^2c \end{bmatrix}
    .
\end{equation}
The second time derivative  $\ddot{Y}$ depends on the inputs ($f_1, f_2$) since 
substituting for the second derivatives of the states from Eq.~\eqref{eq_sys_dyn_isolated} into Eq.~\eqref{eq_yddot_X} results in 
\begin{equation}\label{eq_yddot_F}
    \begin{bmatrix} \ddot{y}_1 \\ \ddot{y}_2 \end{bmatrix} = 
    \begin{bmatrix}
    0 \\ -g 
    \end{bmatrix} + 
        \begin{bmatrix} 0 & \frac{s}{m} \\ 0 & -\frac{c}{m} \end{bmatrix}
    \begin{bmatrix} f_1 \\ f_2 
    \end{bmatrix}
    ~ = 
    \begin{bmatrix}
    0 \\ -g 
    \end{bmatrix} + 
        \tilde{\beta} F
        ,
\end{equation}
where $\tilde{\beta}$ is defined as the matrix of terms preceding the input vector $F = \begin{bmatrix} f_1 & f_2 \end{bmatrix}^T$.
However, the force $F$ cannot be found from Eq.~\eqref{eq_yddot_F} since the matrix $\tilde{\beta}$  is not invertible.
Therefore, assuming that the input ${f}_2$ is sufficiently smooth, and redefining the new input to be $\dot{f}_2$ (with ${f}_2$ considered as an extended state), 
the output expression in Eq.~\eqref{eq_yddot_F} is differentiated again to obtain 
\begin{equation}\label{eq_ydddot_F}
    Y^{(3)} = 
    \begin{bmatrix} {y}^{(3)}_1 \\ {y}^{(3)}_2 \end{bmatrix} = 
    \frac{1}{m}\begin{bmatrix}  f_2\dot{\theta}c \\  f_2\dot{\theta}s \end{bmatrix}
    + \tilde{\beta}
      \begin{bmatrix} f_1 \\ \dot{f}_{2} \end{bmatrix}, 
\end{equation}
where the superscript in brackets $(i)$ indicates the $i^{th}$ time derivative, e.g., $y_k^{(i)}$ denotes $i^{th}$ time derivative of $y_k$ for $k\in\{1,2\}$.
Again, the redefined input ($f_1, \dot{f}_{2}$) cannot be found from Eq.~\eqref{eq_ydddot_F} since the matrix 
$\tilde{\beta}$ is not invertible. 
Therefore, the input is further redefined to be $\ddot{f}_2$, with (${f}_2, \dot{f}_2$) considered as extended states, and 
the output expression in Eq.~\eqref{eq_ydddot_F} is differentiated again to obtain
\begin{equation}\label{eq_y4}
    Y^{(4)} = 
    \begin{bmatrix} y_1^{(4)} \\ y_2^{(4)} \end{bmatrix} = \frac{1}{m}
    \begin{bmatrix} \ddot{f}_2s + 2\dot{f}_2\dot{\theta}c + f_2\ddot{\theta}c - f_2\dot{\theta}^2s \\ -\ddot{f}_2c + 2\dot{f}_2\dot{\theta}s + f_2\ddot{\theta}s + f_2\dot{\theta}^2c \end{bmatrix}
    .
\end{equation}

\noindent 
Substituting  for the second derivatives of the states from Eq.~\eqref{eq_sys_dyn_isolated} into Eq.~\eqref{eq_y4}, and arranging  yields

{\small{
\begin{align}\label{eq:y4_matrix}
 &\begin{bmatrix} y_1^{(4)} \\ y_2^{(4)} \end{bmatrix} ~ = 
     \frac{1}{m} 
    \begin{bmatrix} -\frac{1}{Ml}f_2c^2 & s \\ -\frac{1}{Ml}f_2cs & -c \end{bmatrix}
    \begin{bmatrix} f_1 \\ \ddot{f}_2 \end{bmatrix}
    \\
    ~~~& -\frac{1}{m} \begin{bmatrix}  - 2\dot{f}_2\dot{\theta}c + f_2\dot{\theta}^2s - \frac{1}{Ml}f_2^2sc^2 + \frac{2}{l}f_2\dot{l}\dot{\theta}c + \frac{g}{l}f_2sc \\  - 2\dot{f}_2\dot{\theta}s - f_2\dot{\theta}^2c - \frac{1}{Ml}f_2^2s^2c + \frac{2}{l}f_2\dot{l}\dot{\theta}s + \frac{g}{l}f_2s^2 \end{bmatrix}. \notag
\end{align}
}}

\noindent 
The final redefined input ($f_1, \ddot{f}_{2}$) can be found from Eq.~\eqref{eq:y4_matrix} if the matrix 
preceding the input vector $\begin{bmatrix} f_1 & \ddot{f}_2 \end{bmatrix}^T$ of
\begin{align}
\beta = 
\begin{bmatrix} -\frac{1}{Ml}f_2c^2 & s \\ -\frac{1}{Ml}f_2cs & -c \end{bmatrix}
\end{align}
is invertible. 
The determinant of $\beta$ is $\frac{f_2c}{Ml}$, making it invertible provided the un-stowed 
payload length is nonzero, $l \neq 0$, the payload swing angle does not become horizontal, $\theta \in \left(-\frac{\pi}{2}, \frac{\pi}{2}\right)$, and the payload force 
remains negative, $f_2 < 0$ to ensure that the cable remains taut. 
Given the desired outputs' fourth derivatives, $Y^{(4)}_d$, the redefined input $f_{1}$ and $\ddot{f}_{2}$, can be found by settling the right hand side of  Eq.~\eqref{eq:y4_matrix} to be 
$ \begin{bmatrix} v_1 & v_2 \end{bmatrix}^T$ 
as
\begin{align}\label{eq:F1_ff_0}
    f_{1} &= \frac{Ml}{f_2c}(-m(c v_1 + sv_2) + 2\dot{f}_2\dot{\theta} \notag \\ 
    & \qquad + \frac{1}{Ml}f_2^2sc - \frac{2}{l}f_2\dot{l}\dot{\theta} - \frac{g}{l}f_2s)
    , \\[.52em]
    \label{eq:F2_ff_0}
    \ddot{f}_{2} & = m(sv_1 - cv_2) + f_2\dot{\theta}^2
\end{align}
leading to the system 
\begin{align}
    y_1^{(4)} & =   v_1 
    \label{eq_v_1_total} \\
    y_2^{(4)} & = v_2. 
     \label{eq_v_2_total}
\end{align}

\noindent 
Here, the feedforward inputs ${f}_{1,ff}$ and $\ddot{f}_{2,ff}$ can be found from Eq.~\eqref{eq:F1_ff} and Eq.~\eqref{eq:F2_ff}, respectively, by setting $v_1=y_{1,d}^{(4)}$ and $v_2=y_{2,d}^{(4)}$
\begin{align}\label{eq:F1_ff}
    f_{1,ff} &= \frac{Ml}{f_2c}(-m(c y_{1,d}^{(4)} + sy_{2,d}^{(4)}) + 2\dot{f}_2\dot{\theta} \notag \\ 
    & \qquad + \frac{1}{Ml}f_2^2sc - \frac{2}{l}f_2\dot{l}\dot{\theta} - \frac{g}{l}f_2s)
    , \\[.52em]
    \label{eq:F2_ff}
    \ddot{f}_{2,ff} & = m(sy_{1,d}^{(4)} - cy_{2,d}^{(4)}) + f_2\dot{\theta}^2
\end{align}
and then integrating twice over time to find the feedforward input ${f}_{2,ff}$. A specified trajectory-tracking performance, i.e., a desired characteristic equation for the error dynamics, say 
$$\lambda^4 + \sum_{i=0}^{3} a_{k,i} \lambda^i =0$$
where the error is 
$$e_k = y_k - y_{k,d}, \qquad k \in \{ 1,2 \}$$ can be achieved by selecting the  controller $v_k$, $k \in \{ 1,2 \}$ in Eqs.~\eqref{eq_v_1_total} and \eqref{eq_v_2_total} as
\begin{equation}\label{eq:exact_fb}
    v_k = y_{k,d}^{(4)} 
- \sum_{i=0}^{3} a_{1,i} ( y_k^{(i)} -y_{k,d}^{(i)}), \quad k \in \{ 1,2 \}
\end{equation}
resulting in system inputs 
\begin{align}\label{eq:F1_fb_lin}
    f_1 &= \frac{Ml}{f_2 c} \Big( -m \Big( c \big( y_{1,d}^{(4)} - \sum_{i=0}^{3} a_{1,i} (y_1^{(i)} - y_{1,d}^{(i)}) \big) \notag \\
    & \quad + s \big( y_{2,d}^{(4)} - \sum_{i=0}^{3} a_{2,i} (y_2^{(i)} - y_{2,d}^{(i)}) \big) \Big) \notag \\
    & \quad + 2\dot{f}_2 \dot{\theta} + \frac{1}{Ml} f_2^2 s c - \frac{2}{l} f_2 \dot{l} \dot{\theta} - \frac{g}{l} f_2 s \Big),
\end{align}
\begin{align}\label{eq:F2_fb_lin}
    f_2 &= \iint \Big( m \Big( s \big( y_{1,d}^{(4)} - \sum_{i=0}^{3} a_{1,i} (y_1^{(i)} - y_{1,d}^{(i)}) \big) \notag \\
    & \quad - c \big( y_{2,d}^{(4)} - \sum_{i=0}^{3} a_{2,i} (y_2^{(i)} - y_{2,d}^{(i)}) \big) \Big) + f_2 \dot{\theta}^2 \Big) \, d^2t \notag \\
    & \quad + \dot{f}_2(0)t + f_2(0),
\end{align}
where $f_2(0)$ and $\dot{f}_2(0)$ are initial conditions.
\begin{remark}
    Starting from rest, the initial conditions for Eq.~\eqref{eq:F2_fb_lin} are $f_2(0)=-mg$ and $\dot{f}_2(0)=0$.
\end{remark}

\subsection{Limited state feedback}\label{sec:fbk}

To avoid taking time derivatives of potentially noisy output measurements, the following provides  a state feedback  $F_{fb} = \begin{bmatrix} f_{1,fb} & f_{2,fb} \end{bmatrix}^T$,
that achieves stable trajectory tracking  without these high-order time derivatives  and without full state feedback (i.e., swing angular velocity $\dot{\theta}$), 
provided the tracked trajectories are sufficiently slow, resulting in a small swing angle.

\begin{lemma}
The feedback law 
\begin{equation}\label{eq:feedback}
        F_{fb} = -KX = 
        -
        \begin{bmatrix} 
            k_1 & k_2 & k_3 & k_4 & 0 & 0 \\
            0 & 0 & 0 & 0 & k_5 & k_6
        \end{bmatrix}
        X
        ,
    \end{equation} 
    stabilizes the system in Eq. \eqref{eq_sys_dyn_isolated}
about the equilibrium state, $X_{eq} = \begin{bmatrix} x_0 & 0 & 0 & 0 & l_0 & 0 \end{bmatrix}^T$, and corresponding equilibrium input, $ F_{eq} = \begin{bmatrix} 0 & -mg \end{bmatrix}^T$, 
at any  given cart position $x_0$ and positive payload length $l_0$, 
provided
    \begin{align}
  &   k_1  > 0 , \text{ } k_2l_0-k4  > 0 , \notag \\
  &   \frac{-k_4Mg}{k_2l_0-k_4} + k_1l_0 - k_3 + mg  > 0 , \notag\\
      &   \frac{-k_2k_4Mg + (k_2l_0-k_4)(k_1k_4-k_2k_3+k_2mg)}{-k_4Mg+(k_2l_0-k_4)(k_1l_0-k_3+mg)} > 0 , \notag \\
    & k_5 > 0 , \text{ and } k_6 > 0  
    \label{stability_conditions}
    .
    \end{align}

\end{lemma}

\begin{proof}
Linearization of the system model in Eq. (\ref{eq_sys_dyn_isolated}) about the equilibrium in the lemma results in 
    \begin{equation}\label{eq:linearized_form}
        \dot{X} = AX + BF,
    \end{equation}
where 
applying the input as the feedback $F=F_{fb} $ in Eq.~\eqref{eq:feedback} to Eq. (\ref{eq:linearized_form}) results in the closed-loop dynamics 
    \begin{align}\label{eq:linearized_feedback}
        \dot{X} &=
        \begin{bmatrix}
            \begin{array}{cccc|cc}
            0 & 1 & 0 & 0 & 0 & 0 \\
            -\frac{k_1}{M} & -\frac{k_2}{M} & \frac{mg-k_3}{M} & -\frac{k_4}{M} & 0 & 0 \\
            0 & 0 & 0 & 1 & 0 & 0 \\
            \frac{k_1}{Ml_0} & \frac{k_2}{Ml_0} & \frac{k_3-(M+m)g}{Ml_0} & \frac{k_4}{Ml_0} & 0 & 0 \\
            \hline
            0 & 0 & 0 & 0 & 0 & 1 \\
            0 & 0 & 0 & 0 & -\frac{k_5}{m} & -\frac{k_6}{m} 
            \end{array}
        \end{bmatrix}
        X \notag \\ 
        &=
        \begin{bmatrix}
            A_{1} & 0_{4x2} \\
            0_{2x4} & A_{2}
        \end{bmatrix}
        X.
    \end{align}
    Conditions on the gains in $K$ for stability can be derived from the blocks, $A_1$ and $A_2$, separately.
    The characteristic equation of $A_1$ is given by
    \begin{align}
        &Ml_0\lambda^4 + (k_2l_0-k_4)\lambda^3 + (k_1l_0 - k_3 + (M+m)g)\lambda^2 \notag \\
        &+ k_2g\lambda + k_1g = 0\label{eq:characterisric}
        ,
    \end{align}
    which can be used to construct the Routh array
    \begin{equation}
        \begin{array}{c|ccc}
        \lambda ^4 & Ml_0 & k_1l_0-k_3+(M+m)g & k_1g \\
        \lambda ^3 & k_2l_0-k_4 & k_2g & 0 \\
        \lambda  ^2 & c_1 & k_1g & 0 \\
        \lambda  ^1 & c_2 & 0 & 0 \\
        \lambda ^0 & k_1g & 0 & 0 \\
        \end{array}   
    \end{equation}
    where

    \vspace{-0.1in} 
    \begin{align}
        c_1 
        & 
        = \frac{-k_2Mgl_0 + (k_2l_0-k_4)(k_1l_0-k_3+(M+m)g)}{k_2l_0-k_4}
        \notag\\
        & = \frac{-k_4Mg}{k_2l_0-k_4} + k_1l_0 - k_3 + mg, 
        \notag
             \end{align}   \begin{align}
        c_2 
        & 
        = \frac{-k_1k_2l_0 + k_1k_4 -\frac{k_2k_4Mg}{k_2l_0-k_4}+k_1k_2l_0-k_2k_3+k_2mg}{\frac{-k_4Mg}{k_2l_0-k_4} + k_1l_0 - k_3 + mg}g
        \notag\\
        & = \frac{-k_2k_4Mg + (k_2l_0-k_4)(k_1k_4-k_2k_3+k_2mg)}{-k_4Mg+(k_2l_0-k_4)(k_1l_0-k_3+mg)}g \notag
        .
    \end{align}
    Routh-Hurwitz criteria ensures stability if all terms in the first column of the Routh array have no sign changes, or are positive since the first term $ Ml_0$ is positive.
     The characteristic equation for $A_2$ is given by
        $m\lambda ^2 + k_6 \lambda + k_5 = 0$,
    which results in stability  provided $k_5 > 0$ and $k_6 > 0$. The lemma follows. 
\end{proof}

\begin{remark}
    With full state feedback, poles of $A_1$ can be placed by specifying gains in Eq.~\eqref{eq:characterisric}.
\end{remark}

\begin{corollary}
    The system in Eq. (\ref{eq:linearized_feedback}) can be stabilized without feedback from the swing dynamics (i.e., with $k_3=k_4=0$) by  selecting positive gains $k_1>0$, $k_2>0$, $k_5 > 0$, and $k_6 > 0$.
\end{corollary}
\begin{proof}
    By considering $k_3=k_4=0$, the conditions for the stability of $A_1$ in Eq.~\eqref{stability_conditions} simplify to
    $ k_1 > 0, \text{ } k_2l_0 > 0, \text{ } k_1l_0+mg > 0, \text{ } \frac{k_2}{k_1l_0+mg} > 0$, 
which are satisfied for $k_1>0$, and $k_2>0$.
    %
\end{proof}

\begin{remark}
    With $k_3=k_4=0$, stabilizing feedback is achievable without payload angle measurements, which can ease implementation by only requiring sensors to measure the cart position and pendulum length, e.g., using  encoders  placed on the motors.
\end{remark}

\begin{corollary}
\label{corollary_with_swing_angle} 
    The system in Eq. (\ref{eq:linearized_feedback}) is stable by selecting gains as $k_1>0$, $k_2>0$,  $k_5 > 0$,  $k_6 > 0$  and with the swing dynamics gains as $k_3<mg$,  $k_4 = 0$.
\end{corollary}
\begin{proof}
    By considering $k_4=0$, the conditions for the stability of $A_1$ in Eq.~\eqref{stability_conditions} simplify to
 $          k_1 > 0, \text{ } k_2l_0 > 0, \text{ } k_1l_0 + (mg-k_3) > 0, 
          \frac{k_2(mg-k_3)}{k_1l_0 +  (mg-k_3)} > 0$,
              which are all satisfied for $k_1>0$, $k_2>0$, and $k_3<mg$.
\end{proof}

\begin{remark}
    While not required for stability, the addition of swing-angle feedback (i.e., $k_3<mg$ in Corollary~\ref{corollary_with_swing_angle}) can damp undesired oscillations faster than the case without swing-angle feedback (i.e., $k_3 =0$). The swing-angle feedback requires the use of an  external camera feedback to measure the swing angle $\theta$. 
\end{remark}

The desired state and feedforward satisfies the system dynamics, $$\dot{X}_d = f(X_d) + g(X_d)F_{ff}$$ in  Eq.~\eqref{eq_sys_dyn_isolated}, i.e.,
{\small{
\begin{equation}\label{eq_sys_dyn_desired} 
\frac{d}{dt} X_d
 = 
\begin{bmatrix} 
\dot{x_d} \\
 0 \\
 \dot{\theta}_d \\
 - \frac{2}{l_d}\dot{l}_d\dot{\theta}_d - \frac{g}{l_d}s_d \\
  \dot{l}_d \\
l_d \dot{\theta}_d^2 + g c_d \\
\end{bmatrix}  + 
\begin{bmatrix} 
0 & 0 \\
 \frac{1}{M} &   \frac{-s_d}{M} \\
 0 & 0 \\
  \frac{-c_d}{M l_d} &  \frac{s_d c_d}{M l_d}
  \\
  0 & 0 \\
 \frac{-s_d}{M} & \left( \frac{s_d^2}{M} + \frac{1}{m}\right)
\end{bmatrix} 
\begin{bmatrix} 
 f_{1,ff} \\
 f_{2,ff} 
\end{bmatrix},
\end{equation}
}} 

\noindent 
where $s_d = sin(\theta_d)$ and $c_d = cos(\theta_d)$ and
the corresponding desired system states $X_d$ can be computed in terms of the desired outputs algebraically~\cite{kolar2013}; specifically,
\begin{align} \label{eq:x_d}
    x_d & = y_{1,d} - \frac{\ddot{y}_{1,d}y_{2,d}}{\ddot{y}_{2,d}+g}
    ,
\\
\label{eq:xd_d}
    \dot{x}_d & = \dot{y}_{1,d} - \frac{y_{1,d}^{(3)}y_{2,d}+\ddot{y}_{1,d}\dot{y}_{2,d}}{\ddot{y}_{2,d}+g} - \frac{\ddot{y}_{1,d}y_{2,d}y_{2,d}^{(3)}}{(\ddot{y}_{2,d}+g)^2}
    ,
\\
\label{l_d}
    l_d & = \left(\left(\frac{\ddot{y}_{1,d}y_{2,d}}{\ddot{y}_{2,d}+g}\right)^2 + y_{2,d}^2\right)^{\frac{1}{2}}
    ,
\\
\label{eq:ld_d}
    \dot{l}_d 
    &= 
    \left(\left(\frac{\ddot{y}_{1,d}y_{2,d}}{\ddot{y}_{2,d}+g}\right)^2 + y_{2,d}
    ^2\right)^{-\frac{1}{2}} 
    \Biggl( y_{2,d}\dot{y}_{2,d} + \left( \frac{\ddot{y}_{1,d}y_{2,d}}{\ddot{y}_{2,d}+g} \right) 
    \notag \\
    & \qquad   \times 
    \left( \frac{y_{1,d}^{(3)}y_{2,d}+\ddot{y}_{1,d}\dot{y}_{2,d}}{\ddot{y}_{2,d}+g} - \frac{\ddot{y}_{1,d}y_{2,d}y_{2,d}^{(3)}}{(\ddot{y}_{2,d}+g)^2} \right) \Biggr)
    ,
\\
\label{eq:theta_d}
    \theta_d & = \tan^{-1} \left( \frac{-\ddot{y}_1}{\ddot{y}_2+g} \right)
    .
\\
\label{eq:thetad_d}
    \dot{\theta}_d &= \frac{-y_{1,d}^{(3)}\left( \ddot{y}_{2,d}+g \right) + \ddot{y}_{1,d} y_{2,d}^{(3)}}{(\ddot{y}_{2,d}+g)^2 + \ddot{y}_{1,d}^2}
\end{align}
Eqs.~\eqref{eq:x_d} - \eqref{eq:thetad_d} impose the condition of $\ddot{y}_{2,d} > -g$ to maintain positive cable tension and prevent slackening.

\begin{remark}
\label{rem_trajectory_stability}
The crane robot will follow a four times differentiable desired output trajectory $Y_d$ by applying the feedforward input $F_{ff}$ in Eq.~\eqref{eq_sys_dyn_desired}. Feedback is added to stabilize the desired trajectory in response to perturbations.
\end{remark}

\begin{lemma}
\label{lemma_slow}
The swing dynamics 
($\theta_d$ and $\dot{\theta}_d$), length variation velocity ($\dot{l}_d$), 
and change of feedforward forces from equilibrium values ($F_{ff}-F_{eq}$)
can be made arbitrarily small 
for sufficiently-slowly-varying desired output trajectories $Y_d$ (i.e., for sufficiently-small time derivatives 
$Y_d^{(i)}$ for $i = \{1, \ldots, 4\}$) such that  
\begin{eqnarray} 
\lim_{\max\limits_t \|Y_d^{(i)}(t)\| \to 0 } 
\left( 
\max_t | \theta_d (t) | 
\right) & = 0
\\
\lim_{\max\limits_t \|Y_d^{(i)}(t)\| \to 0 } 
\left( 
\max_t  | 
 \dot{\theta}_d (t)|  \right)& = 0 \\
\lim_{\max\limits_t \|Y_d^{(i)}(t)\| \to 0 } 
\left( 
\max_t | 
 \dot{l}_d(t)|  \right)& = 0 \\
\lim_{\max\limits_t \|Y_d^{(i)}(t)\| \to 0 } 
\left( 
\max_t \| 
F_{ff}(t) - F_{eq}\| \right) & = 0 .
\end{eqnarray}
\end{lemma}
\begin{proof}
Assume the crane robot is commanded a desired output trajectory $Y_d$ from equilibrium (i.e., $F = F_{eq}$).
By Eqs.~\eqref{eq:ld_d}-\eqref{eq:thetad_d}, $\dot{l}_d \rightarrow 0$, $\theta_d \rightarrow 0$, and $\dot{\theta}_d \rightarrow 0$ as $Y_d^{(i)} \rightarrow 0$.
$\theta_d \rightarrow 0 \implies c_d \rightarrow 1$ and $s_d \rightarrow 0$.
Therefore, by Eq.~\eqref{eq:F1_ff}, $f_{1,ff}  \rightarrow f_{1,eq} \rightarrow 0$ and by Eq. \eqref{eq:F2_ff}, $\ddot{f}_{2,ff} \rightarrow 0 \implies f_{2,ff} \rightarrow f_{2,eq} = -mg$.
The lemma follows.
\end{proof}

\begin{lemma}
\label{lemma_stabilization_time_varying_trajectory}
   The origin of the error $ E = X- X_d$  dynamics
    with the feedforward input $F_{ff}$ augmented with feedback $F_{fb}$
    \begin{equation} 
    F 
     = \begin{bmatrix}
        f_{1,ff} \\ f_{2,ff}
    \end{bmatrix} 
    + \begin{bmatrix}
        f_{1,fb} \\ f_{2,fb}
    \end{bmatrix} ~~
    =F_{ff} + F_{fb} 
    \end{equation} 
    is stable provided the 
    desired output trajectories $Y_d$ are sufficiently slowly-varying, i.e., the time derivatives $Y_d^{(i)}$ for $i = \{1, \ldots, 4\}$ are sufficiently small,  
    with sufficiently small deviations in payload length, i.e., 
    $\max\limits_t(|l_d(t)-l_0|)$ is sufficiently small.
\end{lemma}

\begin{proof} 
    The error dynamics are given by
    \begin{equation}
        \dot{E} = \dot{X} - \dot{X}_d
                = f(X) + g(X) F - f(X_d) - g(X_d) F_{ff}
                .
    \end{equation}
    By Taylor Series expansion, 
    \begin{equation}
        f(X) \approx f(X_d) + \left( \frac{\partial f(X)}{\partial X}|_{X=X_d} \right)  (X-X_d)
                ,
    \end{equation}
    \begin{align}
        g(X)F &\approx g(X_d)F_{ff} + \left( \frac{\partial g(X)F}{\partial X}|_{F=F_{ff},X=X_d} \right)(X-X_d) \notag \\
        &\quad + \left( \frac{\partial g(X)F}{\partial F}|_{F=F_{ff},X=X_d} \right)(F-F_{ff})
        .
    \end{align}
    Therefore, the error dynamics becomes
    \begin{align}
        \dot{E} &\approx f(X_d) + \left( \frac{\partial f(X)}{\partial X}|_{X=X_d(t)} \right)  (X-X_d) \notag \\
        &\quad + g(X_d)F_{ff} + \left( \frac{\partial g(X)F}{\partial X}|_{F=F_{ff},X=X_d} \right)(X-X_d) \notag \\
        &\quad + \left( \frac{\partial g(X)F}{\partial F}|_{X=X_d(t),F=F_{ff}} \right)(F-F_{ff}) \notag \\
        &\quad - f(X_d) - g(X_d) F_{ff} \notag \\
        &= f(X_d) - f(X_d) + A(t)E + B(t)F_{fb} \notag \\
        &\quad + g(X_d)F_{ff} - g(X_d)F_{ff} \notag \\
        &= A(t)E + B(t)F_{fb}
        \notag \\
        &= A(t)E - B(t)KE
        \notag \\
        &= A(t)E - B(t)KE + AE - AE + BKE - BKE
        \notag \\
        &= (A - BK)E + (A(t) - A)E + (B - B(t))KE \notag \\
        &= (A - B K) E + P(t)
    \end{align}
    where
    \begin{align}
        A(t)    
            &= \left( \frac{\partial f(X)}{\partial X}|_{X=X_d} \right) 
             + \left( \frac{\partial g(X)F}{\partial X}|_{X=X_d,F=F_{ff}} \right),  \\
        B(t) &= \left( \frac{\partial g(X)F}{\partial F}|_{X=X_d,F=F_{ff}} \right) = g(X_d),
    \end{align}
     and $P(t)$ is a time varying perturbation to the exponentially stable linearized dynamics in Eq.~\eqref{eq:linearized_form} with terms
     \begin{equation}\label{eq:perturbation}
        P(t) = (A(t)-A)E + (B-B(t))KE, 
     \end{equation}
    \begin{align}
        (A(t)-A) =
         \begin{bmatrix}
            0 & 0 & 0 & 0 & 0 & 0 \\
            0 & 0 & - \frac{(mg+f_{2,ff})}{M} & 0 & 0 & 0 \\
            0 & 0 & 0 & 0 & 0 & 0 \\
            0 & 0 & \tilde{c}_1 & -\frac{2}{l_d}\dot{l}_d & \tilde{c}_2 & 0 \\
            0 & 0 & 0 & 0 & 0 & 0 \\
            0 & 0 & \tilde{c}_3 & 2l_d\dot{\theta}_d & \dot{\theta}_d^2 & 0
            \end{bmatrix}E
            ,
    \end{align}
    \begin{equation}
        \tilde{c}_1 = \frac{s_d(f_{1,ff}-s_df_{2,ff})}{Ml_d} + \frac{c_d^2f_{2,ff}l_0 + mgl_d}{Ml_0l_d} + \frac{g(l_d-l_0c_d)}{l_0l_d}
        \notag
        ,
    \end{equation}
    \begin{equation}
        \tilde{c}_2 = \frac{c_d(f_{1,ff}-s_df_{2,ff})}{Ml_d^2} + \frac{gs_d + 2\dot{l}_d\dot{\theta}_d}{l_d^2}
        \notag
        ,
    \end{equation}
    \begin{equation}
        \tilde{c}_3 = -\frac{c_df_{1,ff}}{M} + \frac{2s_dc_df_{2,ff}}{M} - gs_d
        \notag
        ,
    \end{equation}
    \begin{align}
        (B-B(t))KE =        
        \begin{bmatrix}
            0 & 0 \\
            0 & \frac{s_d}{M} \\
            0 & 0 \\
            \frac{l_0c_d-l_d}{Ml_dl_0} & -\frac{s_dc_d}{Ml_d} \\
            0 & 0 \\
            \frac{s_d}{M} & -\frac{s_d^2}{M}
            \end{bmatrix}KE
            .
    \end{align}
    The time varying perturbation $P(t)$ in Eq.~\eqref{eq:perturbation} can be made arbitrarily small for sufficiently-slowly varying desired trajectories (i.e., $\dot{l}_d$, $\theta_d \rightarrow 0$, $\dot{\theta}_d \rightarrow 0$, $f_{1,ff} \rightarrow 0$, and $f_{2,ff} \rightarrow -mg$ 
    by Lemma~\ref{lemma_slow}) with sufficiently small deviations in payload length ($l_d(t)- l_0$).
    As a result, the perturbation $P(t)$ is of the vanishing type as the error $E \to 0$ and satisfies the bound $\|P(t)\| \leq \gamma \|X\|$, 
    where $\gamma$ can be made arbitrarily small for sufficiently-slowly changing desired trajectories. Therefore, 
     by Lemma 9.1 in~\cite{khalil}, 
     the time-varying trajectory $X=X_d$
     is exponentially stable.
\end{proof}

Output tracking can be achieved  by combining the feedforward (from Eq. (\ref{eq:F1_ff}) and Eq. (\ref{eq:F2_ff})) and a stabilizing feedback (from Eq. (\ref{eq:feedback})  with $k_4=0$) as

\vspace{-0.1in} 
\begin{align}\label{eq:F1}
    f_1 &= \frac{Ml_d}{f_2c_d}(-m(c_dy_{1,d}^{(4)} + s_dy_{2,d}^{(4)})
    \notag \\ 
    &+ 2\dot{f}_2\dot{\theta}_d + \frac{1}{Ml_d}f_2^2s_dc_d - \frac{2}{l_d}f_2\dot{l_d}\dot{\theta}_d - \frac{g}{l_d}f_2s_d) 
    \notag \\ 
    &- k_1(x-x_d) - k_2(\dot{x}-\dot{x}_d) - k_3(\theta-\theta_d)
    ,
\\
\label{eq:F2}
    f_2 &=  \iint \left( m\left(s_d y_{1,d}^{(4)} - c_d y_{2,d}^{(4)}\right) + f_2 \dot{\theta}_d^2 \right) \, d^2t \notag \\
        & \quad + \dot{f}_2(0)t + f_2(0) - k_5(l-l_d) - k_6(\dot{l}-\dot{l}_d) 
        .
\end{align}

\vspace{-0.05in} 
\noindent

\subsection{Semi-autonomous controller}

\vspace{-0.05in}
In  semi-autonomous control, the operator's reference command (output velocity, $\dot{Y}_r$) from the joystick interface is used to autonomously plan a snap-continuous (i.e., $C^4$ continuous) trajectory, $Y_d(t)$, designed (i)~to avoid collisions, and (ii)~to be sufficiently smooth for output tracking using the differential flatness property.

\subsubsection{Reference Specification}
The reference output-velocity, $\dot{Y}_r$, is specified  by the operator as 
\begin{equation}\label{eq_Ydot}
\dot{Y}_r~=~
\begin{bmatrix} \dot{y}_{1,r} \\ \dot{y}_{2,r} \end{bmatrix} = 
\begin{bmatrix} \alpha_1j_{1} \\ \alpha_2j_{2}
\end{bmatrix}
,
\end{equation}

\vspace{-0.05in}
\noindent
where $\alpha_{1}$ and $\alpha_{2}$ are gains scaling the joystick inputs, $j_{1}$ and $j_{2}$, respectively.
The velocity reference, $\dot{Y}_r(t)$ at time $t$ is used to define the nominal reference point, $\tilde{Y}_r(t+T)$, to be reached within the time horizon, $T$, from the current time, $t$, i.e., $\tilde{Y}_r(t+T) = Y(t) + \dot{Y}_rT$.
The nominal reference point $\tilde{Y}_r(t+T)$ 
acts similar to velocity commands, as larger operator-inputs ($j_{1}, j_{2}$) will generate higher-speed trajectories.

\subsubsection{Collision Avoidance}
Collision is evaluated through projection from the crane robot~\cite{intention-guided}.
Intersection between obstacles and the projected path of the crane robot, $l_r$, connecting the output, $Y(t)$, to the nominal operator-specified reference point, $\tilde{Y}_r(t+T)$, indicate imminent collision. 
If the projected path falls inside an obstacle, then a corrected reference point, ${Y}_r(t+T)$, is selected to be outside the obstacle. 
Obstacle locations in the manufacturing environment are assumed to be known, allowing obstacles to be modeled by their axis-aligned bounding box.
These bounding boxes extend to the floor of the wing bay as the crane suspension prevents the camera payload from moving below obstacles, illustrated in Fig. \ref{fig:collision}.

An intersection between the projection and an obstacle indicates imminent collision, as the projection predicts that the planned trajectories will pass through the bounding box.
The set $L_w = \{l_{w,0}, l_{w,1}, \ldots, l_{w,N}\}$ of $N$ line segments consisting of the confined space's wall line segments and the obstacles' bounding boxes,  is checked for collision with the projected reference line segment, $l_r$.
The set of intersection points, $P = \{ Y_p \mid Y_p = l_r \cap l_{w,n}, \text{ } n \in \{0, 1, \ldots, N\} \}$, is used to obtain the corrected reference point, $Y_r$, located at a distance offset $\epsilon$ (to account for disturbance) along the projection, $l_r$, towards the current state, $Y$, from the closest intersection point, $Y_p^*$, if there is an intersection (see Fig. \ref{fig:collision}), such that
\begin{equation}\label{eq:corrected_ref}
Y_r =
\begin{cases}
Y_p^* - \epsilon \frac{Y_p^* - Y(t)}{|Y_p^* - Y(t)|}, & \text{if } P \neq \varnothing \text{ (intersection)}, \\
\tilde{Y}_r, & \text{if } P = \varnothing \text{ (otherwise)}.
\end{cases}
\end{equation}
where $Y_p^* = \argmin_{Y_p \in P} \|Y_p - Y(t)\|$.

\begin{figure}[!t]
\centering
\includegraphics[width=0.45\textwidth]{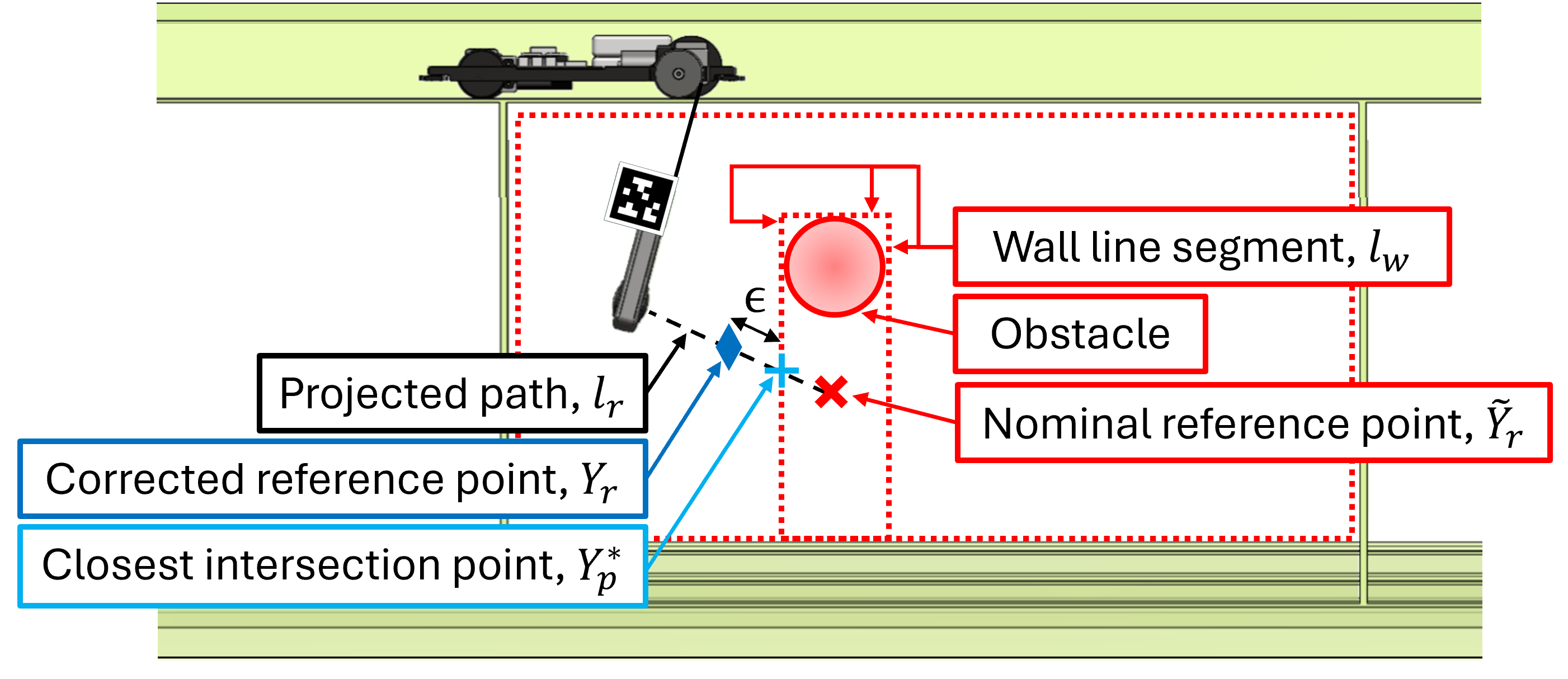}
\caption{
Illustration of collision prevention algorithm when the  nominal reference point (cross), $\tilde{Y}_r$, specified by the operator lies within an obstacle's bounding box.
The corrected reference point (diamond), $Y_r$, is specified at a distance $\epsilon$ towards the current position from the closest intersection point (plus), $Y_p^*$, of the reference line segment (dashed), $l_r$, and the set of bounding box walls (dotted), $L_w$.
}
 \vspace{-0.1in}
\label{fig:collision}
\end{figure}

\subsubsection{Trajectory Generation}\label{sec:traj_gen}
From the corrected reference point, $Y_r$, a snap-continuous, desired output-trajectory $Y_d( \cdot )$ is planned over the time interval $t \le \tau \le t+T$.
Five initial boundary conditions at time $\tau = t$ are found from the outputs and time derivatives of the outputs, $Y_d^{(i)}(t)$, computed from 
\eqref{eq:y},\eqref{eq_ydot_X},\eqref{eq_yddot_F},\eqref{eq_ydddot_F},\eqref{eq:y4_matrix}.
Similarly, five final boundary conditions at time $\tau = t+T$ are defined by the corrected reference point as $Y(t+T) = Y_r$, with final output derivatives set to zero (i.e. $\dot{Y}_d(t+T) = \ddot{Y}_{d}(t+T) = Y^{(3)}_{d}(t+T) = Y^{(4)}_{d}(t+T) = 0$) such that all desired trajectories are planned to reach a resting output state in the case of imminent collision~\cite{intention-guided}.
Moreover, selecting the final time trajectory derivatives to zero (especially, $ \ddot{Y}_{d}(t+T)=0, 
Y^{(3)}_{d}(t+T)=0$) results in zero final swing angle $\theta_d(t+T)=0$ and zero swing-angle velocity
 $\dot{\theta}_d(t+T)=0$, from Eqs.~\eqref{eq:theta_d} and \eqref{eq:thetad_d}, 
and thereby, removes 
residual oscillations at time $t+T$.
The minimal order  polynomial is ninth-order with ten coefficients, i.e., $c_{k,0}$-$c_{k,9}$, to satisfy the  ten boundary conditions on each output trajectory
$Y_{k,d}^{(i)}(t)$ and $Y_{k,d}^{(i)}(t+T)$ for $0\le i\le 4$, $k \in \{1,2\}$. Therefore, 
the desired trajectory for each output,
$y_{k,d}(\tau)$ ($k \in \{1,2\}$), is selected independently as
\begin{equation}\label{eq:poly}
y_{k,d}(\tau) = c_{k,9}\tau^9 + c_{k,8}\tau^8 + \ldots + c_{k,1}\tau + c_{k,0}
.
\end{equation}

\noindent 
Given the desired output trajectory, $Y_d$, as in Eq. \eqref{eq:poly} and its time derivatives found from the polynomials in Eq. \eqref{eq:poly}, the control inputs, $f_1$ and ${f}_2$, can be found from Eq. (\ref{eq:F1}) and Eq. (\ref{eq:F2}).
Fig. \ref{fig:block_diagram} shows a block diagram of the crane robot's control, and
Alg. 1 summarizes the semi-autonomous controller, which generates new polynomial trajectories in real time based on joystick inputs at each control timestep, similar to~\cite{craneOnline}.

\begin{algorithm}[!t]
\caption{Semi-Autonomous Controller}
\begin{algorithmic}[1]\label{alg:semi}
\State \textbf{Given:} Input gains $\alpha_1, \alpha_2$; Time horizon T; Bounding-box line segments $L_w$; Distance offset $\epsilon$; $i \in \{0, \ldots, 4\}$; Cart mass M; Payload mass m; Feedback gains K
\While{\textbf{true}}
    \State $\text{Get joystick inputs } j_1, j_2$
    \State $\text{Get state feedback } X(t) \text{ from external camera}$
    \State $Y_d^{(i)}(t) \gets \text{Eqs. \eqref{eq:y},\eqref{eq_ydot_X},\eqref{eq_yddot_F},\eqref{eq_ydddot_F}},\eqref{eq:y4_matrix}$
    \State $\dot{Y}_r \gets \begin{bmatrix} \alpha_1j_{1} & \alpha_2j_{2} \end{bmatrix}^T$
    \State $\tilde{Y}_r \gets Y(t) + \dot{Y}_r T$
    \State $l_r \gets \overline{Y(t)\tilde{Y}_r}$
    \State $P \gets \{ Y_p \mid Y_p = l_r \cap l_{w,n}, \text{ } n \in \{0, 1, \ldots, N\} \}$
    \State $Y_r \gets \text{Eq. \eqref{eq:corrected_ref}}$
    \State $Y_d^{(i)}(\tau) \gets \text{Eq. (\ref{eq:poly})}$
    \State $X_d(\tau) \gets \text{Eqs. \eqref{eq:x_d}-\eqref{eq:thetad_d}}$
    \State $f_1, f_2 \gets \text{Eqs. (\ref{eq:F1}),(\ref{eq:F2})}$
    \State $\text{Apply } f_1 \text{ and } f_2 \text{ to actuators}$
\EndWhile
\end{algorithmic}
\end{algorithm}

\section{Experimental results and discussion}\label{sec:experiment}

\subsection{Flatness-based input validation}\label{sec_auto_exp}
Compensating for the swing dynamics substantially improves tracking of the desired trajectory when compared to the uncompensated case.
To illustrate, the proposed flatness-based input in Eq. \eqref{eq:F1} and Eq. \eqref{eq:F2} is applied to the fast ramp-like trajectory presented in Section \ref{sec:probform} with a transition time of $T_t = 4$ seconds to compensate for the undesired residual oscillations.
Fig.~\ref{fig:auto_exp}(a) compares the time trajectories of the uncompensated response (neglecting the swing dynamics) to the compensated response (accounting for swing dynamics) for the horizontal $y_1$ and vertical $y_2$ camera positions.
The maximum amplitude of the residual oscillations upon reaching the inspection location are reduced by 89\%, as the oscillation magnitude remains below 1.0 degrees in the compensated case, compared to 9.1 degrees in the uncompensated case, as shown in Fig. \ref{fig:auto_exp}(b).

\begin{figure}[!t]
\centering
\includegraphics[width=0.5\textwidth]{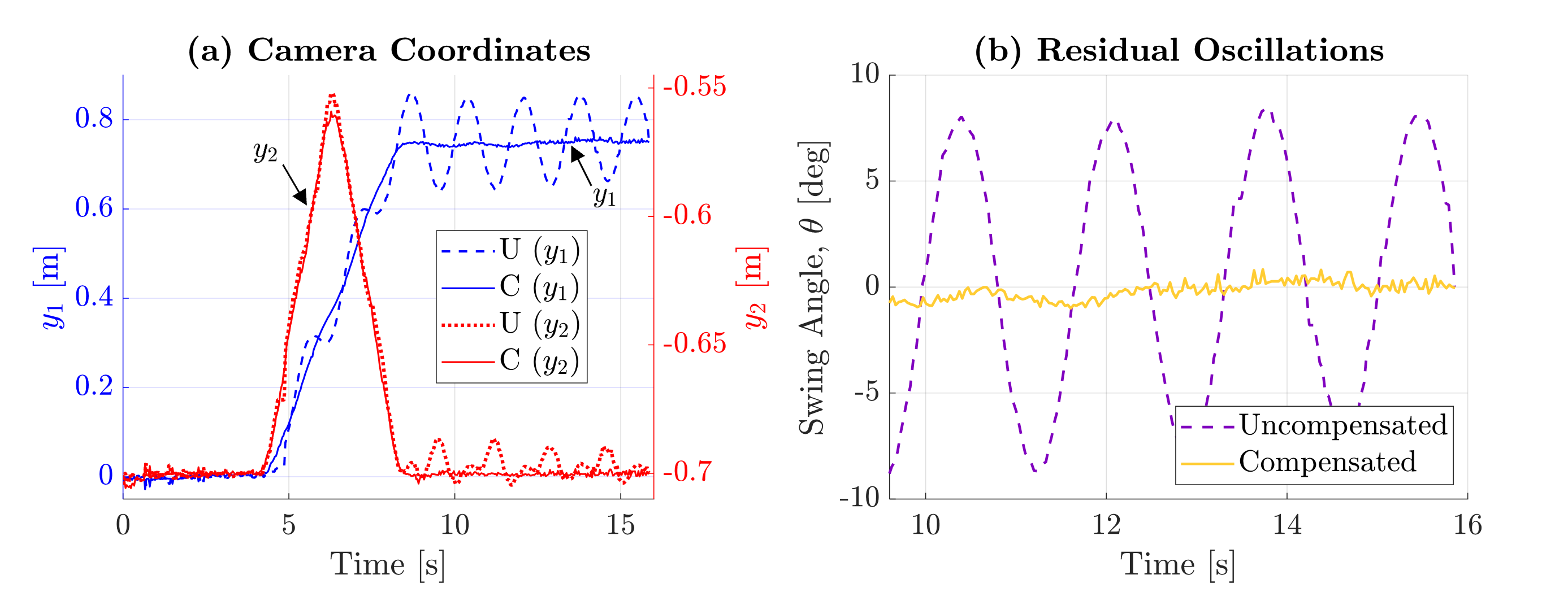}
\caption{
Tracking comparison (see supplementary material for video) of the fast ramp-like trajectory  (transition time of $T_t = 4$ seconds) with 
compensated (C)  and uncompensated (U)  swing dynamics 
for the  ramp-like trajectory described in Section \ref{sec:probform}.  (a)~Horizontal camera position, $y_1$, and vertical camera position $y_2$, and
(b)~residual oscillations after transitioning to the inspection location.
}
\vspace{-0.02151in}
\label{fig:auto_exp}
\end{figure}
\vspace{-0.05in}

\subsection{ Evaluating sensitivity to hyperparameters}\label{sec:e2e_hyperparameter}

Sensitivity of collision avoidance and potential oscillations with the proposed approach is investigated for (a)~varying joystick gain $\alpha_1$ that cause different speeds at which the obstacle is approached and (b)~modeling errors in the cart mass $M$ and the payload mass $m$. In all cases, collision is avoided and the increase in oscillations due to modeling error is removed by augmenting the differentially flat feedforward $F_{ff}$ with feedback $F_{fb}$.  

\subsubsection{Simulation parameters} 
Collision avoidance is evaluated by investigating the horizontal payload response $y_1$ when starting motion from rest and exerting the maximum horizontal joystick command, $(j_1, j_2) = (1, 0)$ for 10 s, towards an obstacle, which has its nearest bounding box wall boundary at $y_1 = 0.75$ m, as shown in Fig.\ref{fig:e2e_hyperparameter}(a).
Nominal parameters are selected to match the human teleoperation experiments. Specifically, 
the cart and camera payload masses are varied from nominal values of $M = 0.815$ kg and $m = 0.225$ kg, respectively. The reference point generation uses a time horizon of $T = 1.5$ s, the scaling gain $\alpha_1$ in the horizontal direction is varied around a nominal value of $\alpha_1=0.12$,   and a distance offset of $\epsilon = 8$ cm is applied for collision avoidance, which match the nominal values in teleoperation experiments.  
Simulated horizontal trajectory responses $y_1$ for different hyperparameters are shown in Fig.~\ref{fig:e2e_hyperparameter}, and discussed below. 

\begin{figure}[!t]
\centering
\includegraphics[width=0.5\textwidth]{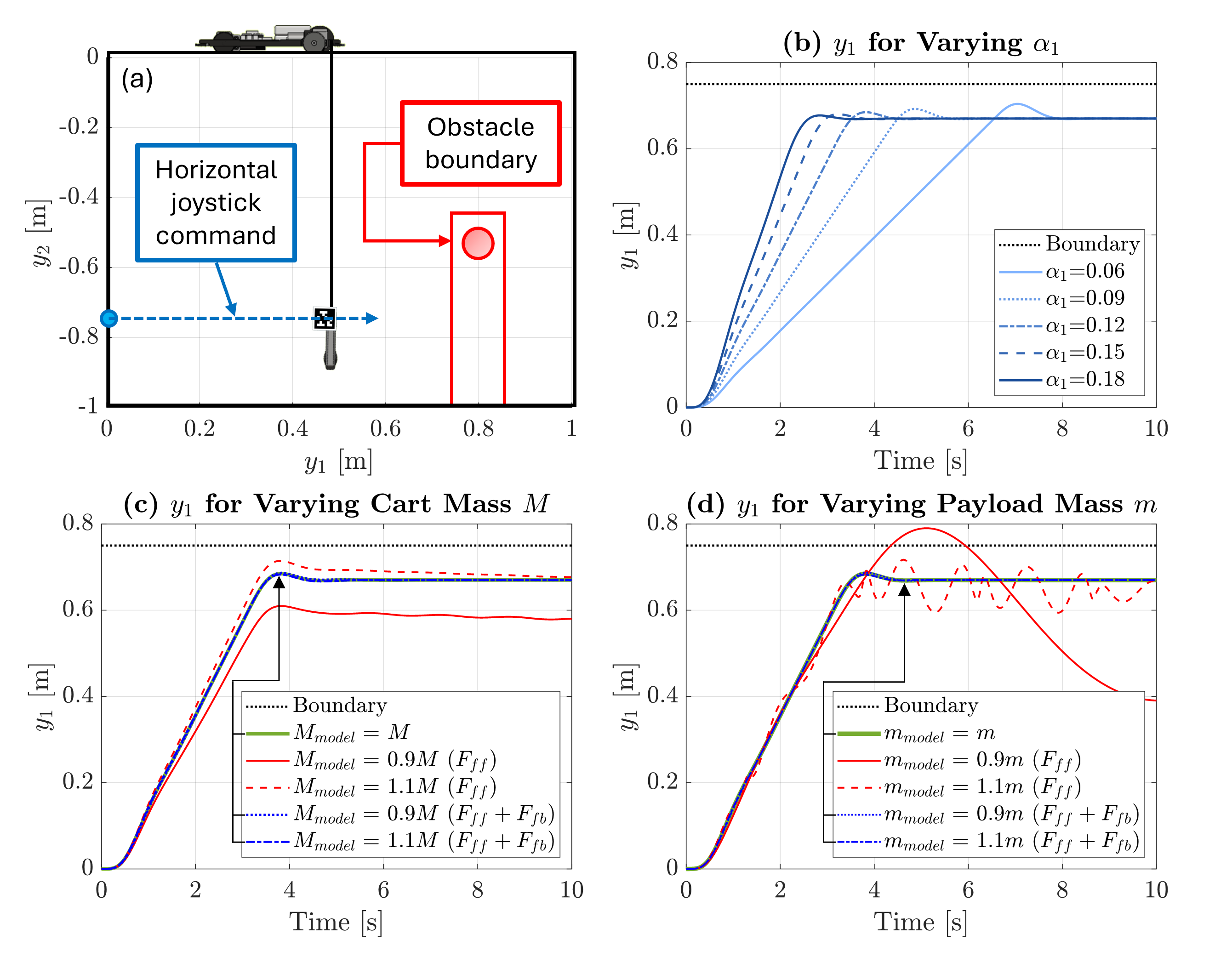}
\caption{
(a) Simulation setup for crane robot where the camera payload approaches an obstacle boundary horizontally.
(b) Horizontal camera position, $y_1$, for varying joystick scaling gain, $\alpha_1$, demonstrating successful collision avoidance with varying speeds due to changes in joystick gain parameter. Feedback $F_{fb}$ reduces sensitivity of feedforward $F_{ff}$ in the horizontal camera position, $y_1$ seen by comparing cases, with and without feedback $F_{fb}$, 
under (c) cart mass modeling errors of $\pm 10 \%$  and  (d) payload mass modeling errors of $\pm 10 \%$. 
}
\label{fig:e2e_hyperparameter}
\end{figure}

\subsubsection{Obstacle avoidance at varying speeds}  
Teleoperation trajectories are governed by joystick scaling gain,  and the time horizon.   
To demonstrate the impact of speed variations, 
the horizontal camera coordinate, $y_1$, responses for five different values of horizontal-axis joystick gain $\alpha_1$ are shown in Fig.~\ref{fig:e2e_hyperparameter}(b). Note that successful collision avoidance is achieved by the algorithm when the joystick parameter is varied.

\subsubsection{Robustness to modeling errors}  
The semi-autonomous controller demonstrates robustness to modeling errors due to its feedback terms.  
The modeled cart mass, $M_\text{model}$, is varied by $\pm 10\%$.  
The response to erroneous flatness-based feedforward input terms, derived from Eqs.~\eqref{eq:F1} and \eqref{eq:F2}, is analyzed by setting all feedback gains to $K = 0$, leading to large errors, as shown in Fig.~\ref{fig:e2e_hyperparameter}(c).  
Introducing feedback gains of $k_1 = 0.01$ N/mm, $k_2 = 0.01$ N/mm$\cdot$s, $k_3 = -0.01$ N/rad, $k_4 = 0$ N/rad$\cdot$s, $k_5 = 0.01$ N/mm, and $k_6 = 0.01$ N/mm$\cdot$s stabilizes trajectories around the expected response.  
The simulation is repeated with the modeled camera payload mass, $m_\text{model}$, varied by $\pm 10\%$, with similar results shown in Fig.~\ref{fig:e2e_hyperparameter}(d).  
Errors from flatness-based feedforward terms demonstrate greater sensitivity to camera payload mass variations compared to cart mass variations, but feedback gains successfully address sensitivity in both cases.

\subsection{User trials}
Evaluated against the conventional industrial gantry crane control approach of decoupled velocity control (VC)~\cite{crane_control} without swing-dynamics compensation, which relies on operator compensation of  oscillations and collision avoidance, the semi-autonomous control (SC) improved efficiency and safety for 12 participants in a fastener inspection task.

\subsubsection{Fastener inspection task}
The fastener inspection task asks participants to move across the confined space over a pipeline obstacle to identify whether three fasteners are properly seated (i.e. no gaps under the fastener), as depicted in Fig. \ref{fig:task}.
Each participant performed three trials; in each trial, the participant performed the task twice in a single-blind manner, once with VC and once with SC.
A pseudorandom number generator specified the order that each controller was used as well as the fastener gap configuration. 
To complete the task, participants recorded which fasteners contained gaps before capturing an image from the camera payload and confirming task completion. 
In the event of collision, the task was recorded as a failed attempt, and the participant restarted the task.
After each task, participants completed a
questionnaire. 

\begin{figure}[!t]
\centering
\includegraphics[width=0.45\textwidth]{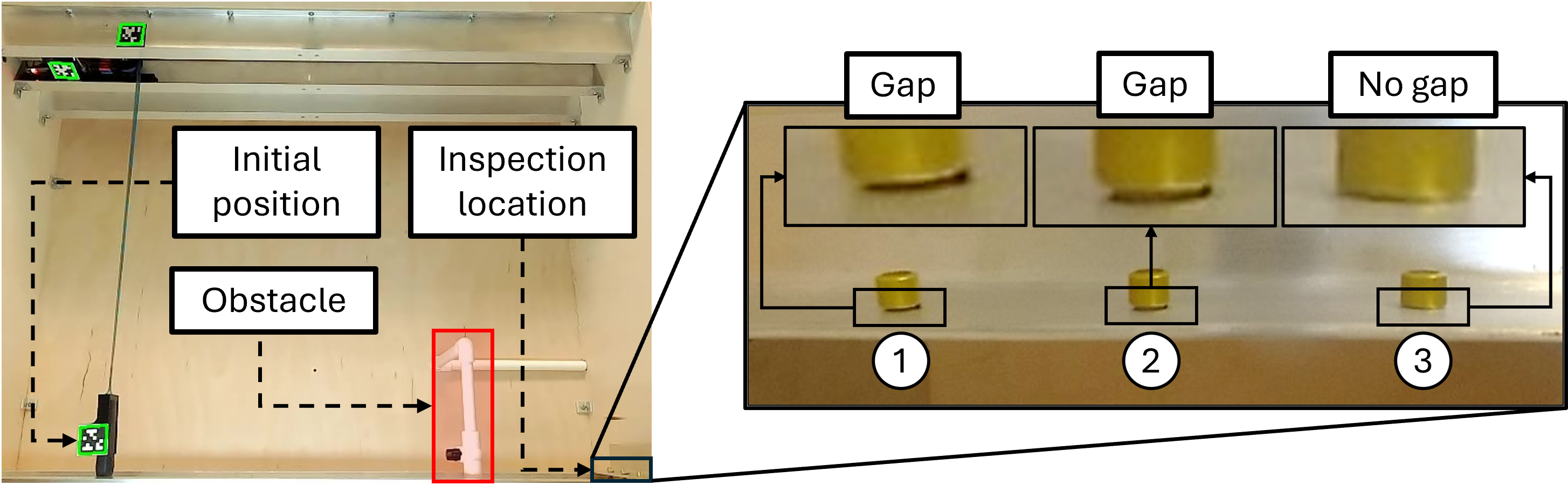}
\caption{
The experimental setup for the fastener inspection task. 
Participants move the crane robot from an initial position over the pipeline obstacle to the inspection location with three fasteners. 
The sample inset image shows  fasteners 1 and 2 with gaps (which can be observed visually), while fastener 3 is properly seated and has no gap.
}
 \vspace{-0.1in}
\label{fig:task}
\end{figure}

\begin{figure*}[!ht]
\centering
\includegraphics[width=\textwidth]{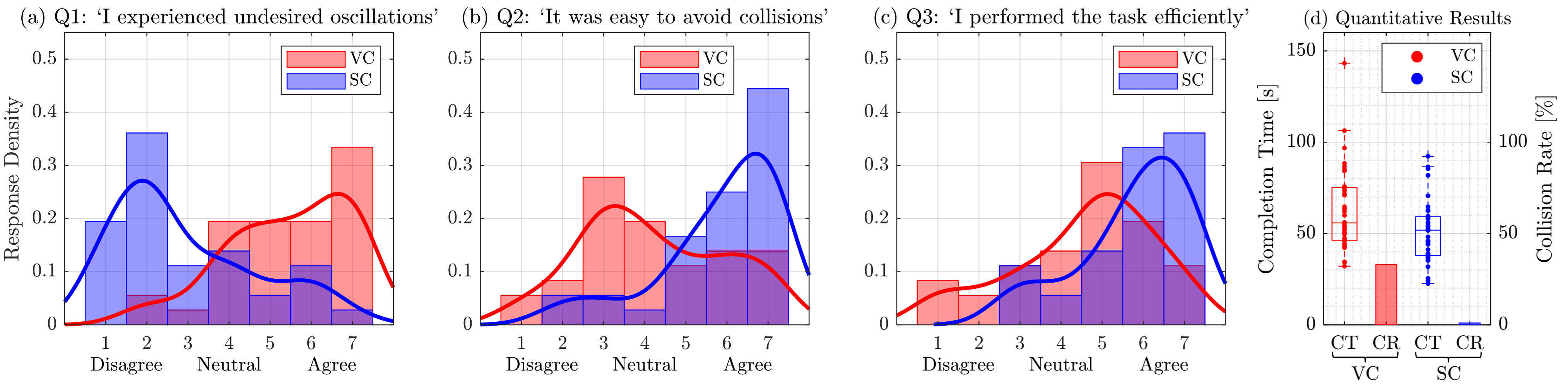}
\caption{
Results from 36 teleoperated fastener inspection trials. (a)-(c) Histogram of response density from the subjective questionnaire from a seven point Likert scale (1: Strongly disagree; 4: Neutral; 7: Strongly agree) comparing responses collected after using VC (without swing-dynamics compensation) and SC (with swing-dynamics compensation). The general trend of response density is represented as a line generated through a squared exponential kernel smoothing with a bandwidth of 0.7. 
(d) Quantitative measures of task completion time (CT) and collision rate (CR) for decoupled velocity control (VC) and semi-autonomous control (SC).
}
\vspace{-0.1in}
\label{fig:teleop_results}
\end{figure*}

\subsubsection{Experiment parameters}\label{sec:exp_params}
Experiments were completed by moving the camera payload from rest at initial output coordinates, $\left(y_1,y_2\right)$, of $(0,-0.72)$ m to inspect three fasteners seated at $(0.88,-0.72)$ m.
The reference velocity, $\dot{Y}_r$ was generated by scaling the joystick inputs, $j_1$ and $j_2$ ranging from $[-1,1]$, by gains $\alpha_1=0.12$ and $\alpha_2=0.04$, respectively, as in Eq. \eqref{eq_Ydot}.
The cart mass, $M$, and camera payload mass, $m$, were 0.815 kg and 0.225 kg, respectively.
VC tracked $\begin{bmatrix}\dot{x} & \dot{l}\text{\,}\end{bmatrix}^T=\dot{Y}_r$ using feedback control, while SC applied Alg. 1 to plan trajectories over a time horizon, $T=1.5$ s, ensuring trajectories remain in the confined space bounded virtually by $y_1\in[0,0.88]$ m and $y_2\in[-0.75,-0.3]$ m and avoided collision with the pipeline obstacle modeled as a bounding box parameterized by $y_1\in[0.575,0.65]$ m and $y_2\in[-0.75,-0.5]$ m using a distance offset of $\epsilon=8$ cm for a conservative buffer.
Feedback gains were experimentally tuned to address perturbations.
Gains $k_1 = 0.8~\frac{\text{N}}{\text{mm}}$ and $k_2 = 0.8~\frac{\text{N}}{\text{mm}\cdot\text{s}}$ for cart position, $x$, and velocity, $\dot{x}$, were increased incrementally until tracking error was reduced without overshoot for cart trajectories (e.g. tracking the horizontal trajectory, $y_1(\tau)$, in Section~\ref{sec:probform}). 
Similarly, gains $k_5 = 0.2~\frac{\text{N}}{\text{mm}}$ and $k_6 = 0.2~\frac{\text{N}}{\text{mm}\cdot\text{s}}$ for camera payload length, $l$, and velocity, $\dot{l}$, were incrementally increased until error was reduced without overshoot when tracking length trajectories (e.g., tracking the vertical trajectory, $y_2(\tau)$, in Section~\ref{sec:probform}).
The gain $k_3 = -0.05~\frac{\text{N}}{\text{rad}}$ for the swing angle, $\theta$, was incrementally increased to suppress disturbed oscillations within a few cycles.

\subsubsection{Reduced oscillation}

Participants reported a decrease in undesired oscillations when using SC compared to VC, as shown through responses to Q1 in Fig. \ref{fig:teleop_results}(a), supported by a Wilcoxon signed rank test, which revealed a statistically significant difference between the two controllers, rejecting the null hypothesis ($p < 0.05$).
Experiencing an increase in undesired oscillations while using VC compared with SC is consistent with the autonomous experiments of Section \ref{sec_auto_exp}, as VC leaves unaccounted residual oscillations for participants while SC leverages the differential flatness to remove oscillations.

\subsubsection{Safer inspection}

The use of SC resulted in safer inspection for both the crane robot and surrounding aircraft structure with a collision rate of 0\% compared to a collision rate of 33\% under VC (Fig. \ref{fig:teleop_results}(d)).
The combination of reduced uncontrolled oscillation and reference point corrections led to participants reporting an ease of avoiding collisions through Q2 while using SC compared with VC (Fig. \ref{fig:teleop_results}(b)).
A Wilcoxon signed rank test confirmed this perception, rejecting the null hypothesis ($p < 0.05$).

\subsubsection{Improved efficiency}
Even neglecting failures due to collisions with VC, SC improved inspection efficiency by reducing mean task completion time by 18.7\% (Fig. \ref{fig:teleop_results}(d)) to 51.2 s compared to a mean task completion time under VC of 63.0 s.
A paired-t test demonstrated a statistically significant difference between responses collected from VC and SC, rejecting the null hypothesis ($p<0.05$). 
Participants also perceived improved task efficiency using SC as opposed to VC through Q3 of the subjective questionnaire (Fig. \ref{fig:teleop_results}(c)), where a Wilcoxon signed rank test confirmed this perception, rejecting the null hypothesis ($p < 0.05$).

\section{Conclusion}\label{sec:conclusion}

This work presented a crane robot for teleoperated in-wing confined space inspection.
To remove undesired oscillations during teleoperation, the swing dynamics of the crane robot are accounted for by exploiting the differentially-flat dynamics to generate sufficiently smooth trajectories for tracking,  while avoiding collision with surrounding obstacles. This enabled  semi-autonomous control with reduced undesired oscillations, eliminated collisions, and enhanced inspection efficiency during teleoperation.
Future work will focus on considering the crane robot's dynamic constraints to plan high-speed, optimal time trajectories while minimizing snap and oscillation during motion within confined spaces.






\section*{ACKNOWLEDGMENT}
The authors thank Shuonan Dong and Jonathan Ahn for guidance on the system design.


\bibliographystyle{ieeetr}
\bibliography{bib}


\end{document}